\documentclass{article} % For LaTeX2e
\usepackage{iclr2015,times}
\usepackage{hyperref}
\usepackage{url}

\usepackage{graphicx}
\usepackage{caption}

\usepackage{subcaption}

\usepackage{pdfpages}

\usepackage{natbib}

\usepackage{amsmath, amsthm, amssymb}
\DeclareMathOperator{\Tr}{Tr}

\usepackage{algpseudocode}
\usepackage{algorithm}

\usepackage{pifont}
\usepackage{float}

\usepackage{hyperref}

\usepackage{textcomp}

\newtheorem{proposition}{Proposition}
\newtheorem{corollary}{Corollary}
\newtheorem{theorem}{Theorem}

\usepackage{color}
\definecolor{DarkOrange}{RGB}{207,102,0}
\definecolor{Blue}{RGB}{51,51,255}

\title{Variance Reduction in SGD by Distributed Importance Sampling}

\author{
Guillaume Alain, Alex Lamb, Chinnadhurai Sankar, Aaron Courville \& Yoshua Bengio\thanks{Yoshua Bengio is a senior CIFAR Fellow}\\
Department of Computer Science and Operations Research \\
Universit\'{e} de Montr\'{e}al\\
Montreal, QC.  H3C 3J7\\
\texttt{\{guillaume.alain, alex.lamb, chinnadhurai.sankar, } \\
\texttt{  aaron.courville\}@umontreal.ca }
}

\iclrconference % Uncomment if submitted as conference paper instead of workshop

\begin{document}

\maketitle

\begin{abstract}
Humans are able to accelerate their learning by selecting training
materials that are the most informative and at the appropriate level of
difficulty.  We propose a framework for distributing deep learning in which
one set of workers search for the most informative examples in parallel
while a single worker updates the model on examples selected by importance
sampling.  This leads the model to update using an unbiased estimate of the
gradient which also has minimum variance when the sampling proposal is
proportional to the L2-norm of the gradient. We show experimentally that
this method reduces gradient variance even in a context where the cost of
synchronization across machines cannot be ignored, and where the factors
for importance sampling are not updated instantly across the training set.
\end{abstract}

\section{Introduction}

Many of the advances in Deep Learning from the past 5-10 years can be attributed
to the increase in computing power brought by specialized hardware (i.e. GPUs).
The whole field of Machine Learning has adapted to this reality, and one of the
latest challenges has been to make good use of multiple GPUs, potentially located on separate computers,
to train a single model.

%Note removed Bengio citation.
One widely studied solution is Asynchronous Stochastic Gradient Descent (ASGD), which is a variation on
SGD in which the gradients are computed in parallel, propagated to a parameter
server, and where we drop certain synchronization barriers to allow the algorithm
to run faster. This method was introduced by~\citet{Bengio-nnlm2003-small} in the
context of neural language models, and extended to model-parallelism
and demonstrated on a large scale by~\citet{Dean-et-al-NIPS2012}.

%Alex's note: this paragraph seems too obvious?
%In a way, ASGD is really the natural thing to try when scaling up. It provides
%a turnkey solution and can be applied with only minor changes to training scripts,
%provided that the ASGD framework and hardware are available and set up for the user.

One of the important limitations of ASGD is that it requires a lot of bandwidth
to propagate the parameters and the gradients.  Moreover, many theoretical guarantees
are lost due to the fact that synchronization barriers are removed and
\emph{stale} gradients are being used. Some theoretical guarantees can still be made
in the context of convex optimization (see ~\citet{Agarwal-2011}, ~\citet{Recht-et-al-NIPS2011}, ~\citet{Lian-2015}), but
any result from convex optimization applied to neural networks (highly non-convex)
has to be used with fingers crossed.

%Alex note: what theoretical guarantees are lost?  Do we have a citation for this?

In this paper we present a different principle for distributed training based
on importance sampling. We demonstrate many interesting theoretical
results, and show some experiments to validate our ideas. The reader should
view these experiments as a proof of concept rather than as an appeal to
switch from Asynchronous SGD to Importance Sampling SGD. In fact,
we can imagine our method being a supplement to ASGD.

Throughout this paper, we will use the word \emph{stale} to refer to the fact that
certain quantities are slightly outdated, but usually not to the point of being
completely unusable. These stale values are usually gradients computed from a set of
parameters $\theta_{t}$ when some reference model is now dealing
with parameters $\theta_{t+\Delta t}$.

In section \ref{sec:scaling-deep-learning} we will explain our distributed Importance Sampling SGD approach.
In section \ref{sec:importance-sampling-theory} we revisit a classical result from the importance sampling literature
and demonstrate a more general result that applies to high dimensions.
We also present a technique that can be used to compute efficiently
the gradient norms for all the individual members of a minibatch.
In section \ref{sec:issgd-in-practice} we discuss
our particular implementation for distributed training.
In section \ref{sec:experiments} we show experiments to illustrate both the reduction in variance
and the increase in performance that it can bring.

%Alex note: why not bring this up when it comes up?  I think that it looks bad if we just say that all our derivations aren't novel.
%In the course of researching this topic, the authors have been made aware
%that some of those theoretical results had already been published,
%or were about to be published imminently. We attribute credit to the other
%people who first published those results, and will nevertheless present
%in this paper our own version of these results, which we hope to make
%more easily accessible.

The main contribution of this paper is to open the door via theoretical and
experimental results to a novel approach to distributed training based on
importance sampling, to focus the attention of the learner on the most
informative examples from the learning point of view.

\section{Scaling Deep Learning by Distributing Importance Sampling}
\label{sec:scaling-deep-learning}

One of the most important constraints on ASGD is the fact that it requires a large
amount of bandwidth. Indeed, all the workers connecting to the parameter server
are required to regularly fetch a fresh copy of the parameters, and all their
computed gradients have to be pushed to the parameter server. For every minibatch
processed by a worker computing a gradient, the memory size of that gradient vector is equal to
the memory size of the parameter vector for the model (i.e. every parameter value gets a gradient value).
Delaying synchronization can result in ``stale'' gradients, that is, gradients
that are computed from a set of parameters that have been fetched from the parameter
server too long ago to be relevant.

The approach that we are taking in this paper is to focus on the most ``useful''
training samples instead of giving equal attention to all the training set.
Humans can learn from a small collection of examples, and a good tutor is
able to pick examples that are useful for a student to learn the current lesson.
This work can therefore be seen as a follow-up on the curriculum learning ideas~\citep{Bengio+al-2009-small},
where the model itself is used to figure out which examples are currently informative for the learner.
The method that we present in this paper will incorporate that intuition
into a training algorithm that is justified by theory rooted in importance sampling~\citep{tokdar2010importance}.

The approach of calibrating the importance sampling coefficients in order to minimize
variance during SGD is also presented by ~\citet{bouchard2015accelerating},
in their method called ``Adaptive Weighted SGD'', in which they adjust the coefficients
by performing an intermediate gradient step to learn the best sampling proposal.
They demonstrate how this can lead to improvements in convergence speed and generalization performance.
In our paper, we show how an exact method can be used to get those optimal coefficients.

Compared to ASGD, our approach can be used to alleviate some of
the communication costs.
%This principle is based on the well-known idea of Importance Sampling , and it can potentially be used to complement ASGD rather
%compete directly against it.
Instead of communicating the gradients on minibatches, the workers communicate
one floating-point number per training sample. In a situation where the parameters
can be of size ranging from 100 MB to 1GB, this cuts down the network transfers
significantly. The parameters still have to be sent on the network to update the workers,
however, but that cost can be amortized over a long period if the algorithm turns
out to be robust to the use of older parameters in order to select the important
samples. Our experiments confirm that hypothesis.

%For this purpose, we present a technique that has been also independently
%discovered (personal communication) by ~\citet{goodfellow2015efficient}
%and posted online earlier than this paper.
%It allows us to compute the gradient
%norms (for dense layers) on all the individual samples within a minibatch.

\section{Importance Sampling in theory}
\label{sec:importance-sampling-theory}

\subsection{Classic case in single dimension}
\label{sec:classic-case-single-dimension}

%We will give here a short summary of importance sampling to be sure that the context
%is clear. Our results presented in the next section \ref{sec:extending-importance-sampling}
%extend this classic result. We would direct the reader to ~\citet{MacKay03} or ~\citet{bishop-book2006}
%to read some more about the topic.

Importance sampling is a technique used to reduce
variance when estimating an integral of the form
\[
\int p(x) f(x) dx = \mathbb{E}_{p(x)}\left[ f(x) \right] \approx \frac{1}{N} \sum_{n=1}^N f(x_n) \hspace{0.5em} \textrm{with} \hspace{0.5em} x_n \sim p(x)
\]
through a Monte-Carlo estimate based on samples drawn from $p(x)$.
Here $f(x)$ can only take on real values, but $x$ can be anything as long as it's compatible with
the probability density function $p(x)$.

It relies on a sampling proposal $q(x)$, for which $0<q(x)$ whenever $0<p(x)$,
and the observation that
\begin{equation}
  \label{eqn:classic-importance-sampling-expectations}
\mathbb{E}_{p(x)}\left[ f(x) \right] = \mathbb{E}_{q(x)}\left[ \frac{p(x)}{q(x)} f(x) \right].
\end{equation}
Since all the quantities in the following empirical sum are independent,
\[
 \frac{1}{N} \sum_{n=1}^N \frac{p(x_n)}{q(x_n)} f(x_n) \hspace{0.5em} \textrm{with} \hspace{0.5em} x_n \sim q(x)
\]
we can directly verify they are unbiased and then try to minimize their variance.
The unbiasedness follows directly from equation (\ref{eqn:classic-importance-sampling-expectations}),
and with a little work can prove that that the variance is minimized when
\begin{equation}
\label{eqn:classic-importance-sampling-optimal-q}
q^*(x) \propto p(x) \left| f(x) \right|.
\end{equation}

%We are not going to reproduce that demonstration here, because it can be found
%rather easily online, and because we demonstrate a more general result
%in this paper. The classic result is simply a consequence of
%Theorem \ref{th:multi-dim-imp-samp} in the particular case of one dimension.

\subsection{Extending beyond a single dimension}
\label{sec:extending-importance-sampling}

In this section we generalize the classic importance sampling
to allow the function $f$ to take values in $\mathbb{R}^d$. The result referenced
here as Theorem \ref{th:multi-dim-imp-samp} is contained in the work of
~\citet{zhao2014stochastic}, but it is stated there without proof,
and it is embedded in their specific context.
%We suspect that it is also known in the folklore, but we failed to find out
%where in the statistics literature this might be.

Minimizing the variance is a well-defined objective in one dimension,
but when going to higher dimensions we have to decide what we would like
to minimize.

For our application, a natural choice of objective function ~\citep{bouchard2015accelerating}
would be the trace of the covariance matrix of the proposal distribution, $\Tr(\Sigma)$,
because it corresponds to the sum of all the eigenvalues of $\Sigma$, which is a positive semi-definite matrix.
It also corresponds to sum of all the variances for each individual component
of the gradient vector. We can also imagine minimizing $\left\Vert\Sigma(q)\right\Vert_{\textrm{F}}^{2}$,
but in this case this would yield a different $q^*$ for which we do not
know of an analytical form.

A nice consequence of our choice is that, when $d=1$, this $\Tr(\Sigma)$ will get back the classic result
from the importance sampling literature.
This is an pre-requisite for any general result.

\vspace{1em}

\begin{theorem}
\label{th:multi-dim-imp-samp}

Let $\mathcal{X}$ be a random variable in $\mathbb{R}^{d_1}$ and $f(x)$ be any function
from $\mathbb{R}^{d_1}$ to $\mathbb{R}^{d_2}$.
Let $p(x)$ be the probability density function of $\mathcal{X}$, and let $q(x)$ be a valid
proposal distribution for importance sampling with the goal of estimating
\begin{equation}
\mathbb{E}_{p} \left[ f(x) \right] = \int p(x) f(x) dx = \mathbb{E}_{q}\left[ \frac{p(x)}{q(x)} f(x) \right].
\end{equation}

The context requires that $q(x) > 0$ whenever $p(x) > 0$. We know that the
importance sampling estimator
\begin{equation}
\frac{p(x)}{q(x)} f(x) \hspace{1em} \textrm{with $x \sim q$}
\end{equation}
has mean $\mu = \mathbb{E}_{p}\left[ f(x) \right]$ so it is unbiased.

Let $\Sigma(q)$ be the covariance of that estimator, where we include $q$
in the notation to be explicit about the fact that it depends on the choice of
$q$.

Then the trace of $\Sigma(q)$ is minimized by the following optimal proposal $q^*$ :
\begin{equation}
q^*(x) = \frac{1}{Z} p(x) \left\Vert f(x) \right\Vert_2 \hspace{0.5em} \textrm{where} \hspace{0.5em}
      Z = \int p(x) \left\Vert f(x) \right\Vert_2 dx
\end{equation}
which achieves the optimal value
\[
\Tr(\Sigma(q^*)) = \left(\mathbb{E}_{p}\left[\left\Vert f(x)\right\Vert _{2}\right]\right)^{2} - \left\Vert \mu \right\Vert_2^2.
\]
\end{theorem}
\begin{proof}
See appendix section \ref{appsec:extending-importance-sampling}.
\end{proof}
Note that in theorem \ref{th:multi-dim-imp-samp}
we refer to a general function $f$.
It should be understood by the reader that we are really interested
in the particular situation in which $f$ represents the gradient of
a loss function with respect to the parameters of a model to be trained.
However, since our results are meant to be more general than that,
we tried to avoid contaminating them with those specific details,
and decided to stick with $f(x)$ instead of talking about $\nabla_{\theta} \mathcal{L}(x_n)$.

Also, as a side-note, some readers would feel that it is strange to be taking
the integral of a vector-valued function $f(x)$, but we would like
to remind them that this is always what happens when we consider
the expectation of a random variable in $\mathbb{R}^2$.

\vspace{1cm}

From theorem \ref{th:multi-dim-imp-samp} we can get the following
corollary \ref{co:trace-covariance-imp}. Here we introduce the notation $\tilde{\omega}_n$ to refer
to un-normalized probability weights used in importance sampling (along with their
normalized equivalents $\omega_n$), which we are going to need later in the paper.
In corollary \ref{co:trace-covariance-imp}, we do not assume that the probability weights
are selected to be norms of gradients, but this is how they are going
to be used throughout section \ref{sec:issgd-in-practice}.

\begin{corollary}
\label{co:trace-covariance-imp}
Using the context of importance sampling as described in theorem \ref{th:multi-dim-imp-samp}, let $q(x)$
be a proposal distribution that is proportional to $p(x)h(x)$ for some
function $h:\mathcal{X}\rightarrow\mathbb{R}^+$.
As always, we require that $h(x) > 0$ whenever $f(x) > 0$.

Then we have that the trace of the covariance of the importance sampling estimator is given by
\[
\Tr(\Sigma(q)) = \left( \int p(x) h(x) dx \right) \left( \int  p(x) \frac{ \left\Vert f(x) \right\Vert_2^2 }{h(x)} dx \right) - \left\Vert \mu \right\Vert_2^2,
\]
where $\mu = \mathbb{E}_{p(x)}\left[ f(x) \right]$.
Moreover, if $p(x)$ is not known directly, but we have access to a dataset $\mathcal{D} = \{x_n\}_{n=1}^{\infty}$ of samples
drawn from $p(x)$, then we can still define $q(x) \propto p(x) h(x) $ by associating the probability
weight $\tilde{\omega}_n = h(x_n)$ to every $x_n \in \mathcal{D}$.

To sample from $q(x)$ we just normalize the probability weights
\[
\omega_n = \frac{\tilde{\omega}_n}{ \sum_{n=1}^N \tilde{\omega}_n }
\]
and we sample from a multinomial distribution with argument $(\omega_1,\ldots,\omega_N)$ to
pick the corresponding element in $\mathcal{D}$.

In that case, we have that
\begin{eqnarray*}
\Tr(\Sigma(q)) & = & \left( \frac{1}{N} \sum_{n=1}^N \tilde{\omega}_n \right) \left( \frac{1}{N} \sum_{n=1}^N \frac{ \left\Vert f(x_n) \right\Vert_2^2 }{\tilde{\omega}_n} \right) - \left\Vert \mu \right\Vert_2^2 \\
               & = & \left( \frac{1}{N} \sum_{n=1}^N \omega_n \right) \left( \frac{1}{N} \sum_{n=1}^N \frac{ \left\Vert f(x_n) \right\Vert_2^2 }{\omega_n} \right) - \left\Vert \mu \right\Vert_2^2.
\end{eqnarray*}

%\begin{eqnarray}
%\Tr(\Sigma(q)) & = & \left( \frac{1}{N} \sum_{n=1}^N \tilde{\omega}_n \right) \left( \frac{1}{N} \sum_{n=1}^N \frac{ \left\Vert f(x_n) \right\Vert_2^2 }{\tilde{\omega}_n} \right) \\
%               & = & \left( \frac{1}{N} \sum_{n=1}^N \omega_n \right) \left( \frac{1}{N} \sum_{n=1}^N \frac{ \left\Vert f(x_n) \right\Vert_2^2 }{\omega_n} \right).
%\end{eqnarray}

\end{corollary}
\begin{proof}
See appendix section \ref{appsec:extending-importance-sampling}.
\end{proof}

%\subsection{Why lower variance matters}
%The motivation behind the results from the preceding section \ref{sec:extending-importance-sampling}
%is that we want to apply importance sampling to SGD to train neural networks.
%This is a situation where the gradients are high-dimensional and
%the regular SGD procedure has to deal with very estimates with large variance.

%We refer to the \emph{true gradient} as the empirical expectation of the gradient $g(x)$
%taken over the training set.
%For every training minibatch, we get an estimate of the true gradient,
%and we use that to make one step of training. This estimate is unbiased,
%but it goes without saying that a lower variance is generally a desirable property.
%Less variance translates into less training time.

%There is already existing work from ~\citet{bouchard2015accelerating} who use adaptive methods
%to craft a good sampling proposal $q(x)$ that yields a lower variance.
%They opted for the trace of $\Sigma$ as an objective to minimize,
%where $\Sigma$ is the covariance matrix of the gradient $g(x)$ when
%sampled from the training set.

\subsection{Dealing with minibatches}
\label{sec:dealing-minibatches}

To apply the principles of ISSGD,
we need to be able to evaluate $\left\Vert g(x_n) \right\Vert_2$
efficiently for all the elements of the training set, where $g$ here is the gradient of the loss
with respect to all the parameters of the model.

In the current landscape of machine learning, using minibatches is a fact of life.
Any training paradigm has to take that into consideration, and this can be a challenge
when one considers that the gradient for a single training sample is as big
the parameters themselves. This fact is generally not a problem since the
gradients are aggregated for all the minibatch at the same time, so the cost of
storing the gradients is comparable to the cost of storing the model parameters.

In this particular case, what we need is a recipe to compute the gradient norms
directly, without storing the gradients themselves. The recipe in question,
formulated here as proposition \ref{prop:individual-grad-norms},
was published by ~\citet{goodfellow2015efficient} slightly prior to our work.
It applies to the fully-connected layers, but unfortunately not to convolutional layers.

\begin{proposition}
\label{prop:individual-grad-norms}
Consider a multi-layer perceptron (MLP) applied to minibatches of size $N$,
and with loss $\mathcal{L} = \mathcal{L}_1 + \ldots + \mathcal{L}_N$, where
$\mathcal{L}_n$ represents the loss contribution from element $n$ of the minibatch.

%\vspace{1em}

Let $(W,b)$ be the weights and biases at any particular fully-connected layer so that
$X W + b = Y$, where $X$ are the inputs to that layer and $Y$ are the outputs.

%\vspace{1em}

The gradients with respect to the parameters are given by
\begin{eqnarray*}
\frac{\partial \mathcal{L}}{\partial W} & = & \frac{\partial \mathcal{L}_1}{\partial W} + \ldots + \frac{\partial \mathcal{L}_N}{\partial W} \\
\frac{\partial \mathcal{L}}{\partial b} & = & \frac{\partial \mathcal{L}_1}{\partial b} + \ldots + \frac{\partial \mathcal{L}_N}{\partial b}
\end{eqnarray*}
where the values $\left( \frac{\partial \mathcal{L}_n}{\partial W}, \frac{\partial \mathcal{L}_n}{\partial b} \right)$
refer to the particular contributions coming from element $n$ of the minibatch.
\vspace{1em}
Then we have that
\begin{eqnarray*}
%\left\Vert \frac{\partial \mathcal{L}_n}{\partial W} \right\Vert_\textrm{F}^2 & = & \left\Vert X[n,:] \vspace{1em} \right\Vert_2^2 \hspace{1em} \left\Vert \frac{\partial \mathcal{L}}{\partial Y}[n,:] \right\Vert_2^2 \\
\left\Vert \frac{\partial \mathcal{L}_n}{\partial W} \right\Vert_\textrm{F}^2 & = & \Vert X[n,:] \Vert_2^2  \hspace{0.5em} \cdot \hspace{0.5em} \left\Vert \frac{\partial \mathcal{L}}{\partial Y}[n,:] \right\Vert_2^2 \\
\left\Vert \frac{\partial \mathcal{L}_n}{\partial W} \right\Vert_2^2 & = & \left\Vert \frac{\partial \mathcal{L}}{\partial Y}[n,:] \right\Vert_2^2 \\
\end{eqnarray*}
where the notation $X[n,:]$ refers to row $n$ of $X$, and similarly for $\frac{\partial \mathcal{L}}{\partial Y}[n,:]$.

That is, we have a compact formula for the Euclidean norms of the gradients of the parameters, evaluated for each $N$ elements of the minibatch independently.
\end{proposition}
\begin{proof}
See appendix section \ref{appsec:dealing-minibatches}.
\end{proof}

%We have to highlight the fact that this proposition contrasts
%with the usual way to compute the gradients for a minibatch in a way that sums
%(or averages) all the independent contributions and makes it impossible to determine
%the relative contribution of each minibatch element.

Note that proposition \ref{prop:individual-grad-norms} applies to MLPs that
have all kinds of activation functions and/or
pooling operations, as long as the parameters $(W,b)$ are not shared between layers.
We can ignore the activation functions when applying proposition \ref{prop:individual-grad-norms}
because the activation functions do not have any parameters, and the linear operation part
(matrix multiplication plus vector addition) simply uses whatever quantities
are backpropagated without knowing what comes after in the sequence of layers.

Despite the fact that convolutions are linear operations (in the mathematical sense),
this formula fails to apply to convolutions because of their sparsity patterns and
their parameter sharing.

In a situation where we face convolutional layers along with fully-connected layers,
proposition \ref{prop:individual-grad-norms} applies to the fully-connected layers.
For our purpose of performing
importance sampling, this is not satisfying because we would have to find another
way to compute the gradient norms for all the parameters. One might suggest to
abandon the plan of achieving \emph{optimal} importance sampling and simply ignore
the contributions of sparsely-connected layers, but we do not investigate this
strategy in this paper.

\section{Distributed implementation of ISSGD}
\label{sec:issgd-in-practice}

%% Editing Note : I'm no longer sure that we need to have this section
%%                that clarifies the terminology.
%\subsection{Terminology}
%
%It is worth spending some time on the terminology here.
%The value $\tilde{\omega}_n$ is going to be used to indicate,
%for the corresponding sample $x_n$ from the training set, how probable
%it is that $x_n$ is selected to be used in a minibatch by ISSGD.
%We are going to normalize the $\tilde{\omega}_n$ before using them,
%but they are going to be stored, pre-normalization, in the database,
%so we still need a way to refer to them.
%
%In this paper, they happen to be defined to be equal to the gradient norms,
%but we want to abstract this away because we could have used any other quantity.
%Therefore, we will refer to the $\tilde{\omega}_n$ as \emph{probability weights},
%and leave the tilde on when they are not yet normalized.
%
%As much as we would have liked to use \emph{importance weight} to refer to them,
%there is already a tradition of using that term to refer to another quantity
%in importance sampling.

\subsection{Using an oracle to train on a single machine}

Assume for a moment that we are training on a single machine, and that
we have access to an oracle that could instantaneously evaluate $\tilde{\omega}_n = \left\Vert g(x_n)\right\Vert_2$ on all
the training set, then it is easy to implement importance sampling in an exact fashion.
These $\tilde{\omega}_n$ depend on the model parameters currently sitting on the GPU,
but we assume the oracle is nevertheless able to come up with the values.

We compose minibatches of size $M$ based on a re-weighting of the training set by sampling
(with replacement) the values of $x_n$ with probability proportional to $\tilde{\omega}_n$.
Let $\left(i_1,\ldots,i_M\right)$ be the indices sampled to compose that minibatch,
that is, we are going to use samples $\left(x_{i_1},\ldots,x_{i_M}\right)$ to
perform forward-prop, backward-prop, and update the parameters.
We have to scale the loss accordingly, as prescribed by importance sampling,
so we end up using the loss
\[
    \mathcal{L}_\theta(\textsc{minibatch}) = \left(\frac{1}{N}\sum_{n=1}^N\tilde{\omega}_n\right) \frac{1}{M} \sum_{m=1}^M \frac{1}{\tilde{\omega}_{i_m}} \mathcal{L}_\theta(x_{i_m}).
\]
As as sanity check, we can see that this falls back to the usual value of
$\frac{1}{M}$ if we find ourselves in a situation where all the $\tilde{\omega}_n$
are equal, which corresponds to the situation where the minibatch is composed
from samples from the training set selected uniformly at random.

We can see from corollary \ref{co:trace-covariance-imp} that the expected trace of the covariance
matrix over the whole training set is given by
\begin{equation}
\Tr(\Sigma(q)) = \left( \frac{1}{N} \sum_{n=1}^N \tilde{\omega}_n \right) \left( \frac{1}{N} \sum_{n=1}^N \frac{ \left\Vert g(x_n) \right\Vert_2^2 }{\tilde{\omega}_n} \right) - \left\Vert g^\textsc{true} \right\Vert_2^2.
\end{equation}
The constant $\left\Vert g^\textsc{true} \right\Vert_2^2$ does not depend on the
choice of $q$ so we will leave it out of the current discussion.
Refer to section \ref{appsec:gtrue_norm} for more details about it.

The oracle allows us to achieve the \emph{ideal} Importance Sampling SGD, and this quantity becomes
\begin{equation}
  \label{eqn:q-ideal-trcov}
\Tr(\Sigma(q_{\textsc{ideal}})) = \left( \frac{1}{N} \sum_{n=1}^N \tilde{\omega}_n \right)^2 - \left\Vert g^\textsc{true} \right\Vert_2^2.
\end{equation}

In this situation, we are using $q_{\textsc{ideal}}$ as notation instead of $q^*$.
This is because we will want to contrast this situation with $q_{\textsc{unif}}$
and $q_{\textsc{stale}}$ that we will define shortly.

When performing SGD training with $q_{\textsc{ideal}}$, we can plot those values
of equation (\ref{eqn:q-ideal-trcov}) as we go along, and we can compare at each time step
the $\Tr(\Sigma(q_{\textsc{ideal}}))$ with the value of $\Tr(\Sigma(q_{\textsc{unif}}))$
that we would currently have if we were using uniform sampling to construct
the minibatches. The latter are given by
\begin{equation}
  \label{eqn:q-unif-trcov}
\Tr(\Sigma(q_{\textsc{unif}})) = \frac{1}{N} \sum_{n=1}^N \left\Vert g(x_n) \right\Vert_2^2 - \left\Vert g^\textsc{true} \right\Vert_2^2.
\end{equation}
In figure \ref{fig:sqrtcov} we will see those quantities compared during an experiment where
we do not have access to an oracle, but where we can still evaluate what would
have been the $\Tr(\Sigma(q_{\textsc{ideal}}))$ that we would have had if we
had an oracle. This is relatively easy to evaluate by using equation (\ref{eqn:q-ideal-trcov}).

\subsection{Implementing the oracle using multiple machines}

In practical terms, we can implement a close approximation to that oracle
by throwing more computational resources at the problem. One machine is selected
to be the \emph{master}, running ISSGD, and it will query
a database in order to read the probability weights $\tilde{\omega}_n$.
Computing those probability weights and pushing them on the database
is the job of a collection of \emph{workers}.
The workers spend all their
lifetimes doing two things : getting recent copies of the parameters from the database,
and updating the values of $\tilde{\omega}_n$ to keep them as fresh as possible.
The master communicates its current model parameters to the database as regularly
as possible, but it tries to do a non-trivial amount of training in-between.

Both the master and the workers have access to a GPU to process minibatches,
but only the master will update the parameters (through ISSGD).
The master is almost oblivious to the existence of the workers. It communicates its
parameters to the database, and it gets a set of probability weights $\{\tilde{\omega}_n\}_{n=1}^N$
whenever it asks for them.

The presence of the database between the master and the workers allows the master
to ``fire and forget'' its parameter updates. They are pushed to the database, and
the workers will retrieve them when they are ready to do so. The same goes for the
workers pushing their updates for the $\tilde{\omega}_n$. They are not communicating
directly with the master. Among other things, this also allows for us to potentially
use any database tool to get more performance (e.g. sharding), but we currently
are not doing anything fancy in that regards.

Because of the various costs and delays involved in the system, the master should
not expect the values of $\tilde{\omega}_n$ to be perfectly up-to-date.
That is, the master has parameters $\theta_{t+\Delta_t}$ on its GPU, but it is receiving weights
$\tilde{\omega}_n$ that are based on parameters $\theta_{t}$.
We refer to those weights as being \emph{stale}.

There are degrees of staleness,
and the usefulness of a weight computed 5 minutes ago differs greatly from that of
a weight computed 2 seconds ago. We refer to $q_{\textsc{stale}}$ as the
proposal that is based on all the weights from the previous iteration.
It serves its role as pessimistic estimator, which is generally worse than what we are actually using.
It is also easier to compute because we can get it
from values stored in the database without having to run the model on anything more.

Our evaluation of $\Tr(\Sigma(q_{\textsc{stale}}))$ is based on assuming
that all the probability weights come from the previous set of parameters, so they
are certain to be outdated. Now it becomes a question of how much we are hurt by staleness.
Without trying to introduce too much notation, for the next equation
we will let $\tilde{\omega}_n^{\textsc{old}}$ refer to the outdated weights
at a given time. Then we have that
\begin{equation}
  \label{eqn:q-stale-trcov}
\Tr(\Sigma(q_{\textsc{stale}})) = \left( \frac{1}{N} \sum_{n=1}^N \tilde{\omega}_n^{\textsc{old}} \right) \left( \frac{1}{N} \sum_{n=1}^N \frac{ (\tilde{\omega}_n)^2 }{\tilde{\omega}_n^{\textsc{old}}} \right) - \left\Vert g^\textsc{true} \right\Vert_2^2.
\end{equation}

We know for a fact that $\Tr(\Sigma(q_\textsc{ideal}))$ is the lower bound on
all the possible $\Tr(\Sigma(q))$.
When the weights are not in a horrible state due to excessive staleness,
we generally observe experimentally that the following inequality holds:
\[
\Tr(\Sigma(q_\textsc{ideal})) \leq \Tr(\Sigma(q_\textsc{stale})) \leq \Tr(\Sigma(q_\textsc{unif})).
\]
This has proven to be verified for the all the practical experiments that
we have done, and it does not even depend on the training being done
by importance sampling. Again, this is not an equality that holds all the
time, and setting the probability weights $\tilde{\omega}_n$ to be randomly generated values
will break that inequality.

\subsection{Exact implementation vs relaxed implementation}

To illustrate the whole training mechanism, we start in figure \ref{fig:master-database-worker} by showing
a \emph{synchronized} version of ISSGD. In that diagram,
we show the database in the center, and we have horizontal dotted lines to
indicate where we would place synchronization barriers. This would happen
after the master sends the parameters to the database, because it can decide
then to wait for the workers to update all the probability weights $\tilde{\omega}_n$.
The workers themselves could work their way through all the training set
before checking for recent parameter updates present on the database.
This is rather excessive, which is why we use those
synchronization barriers only to perform sanity checks, or to study the
properties that ISSGD if it was performed in the complete
absence of staleness.

This kind of relaxation is analogous to how ASGD discards the synchronization barriers
to trade away correctness to gain performance. However, in the case of ISSGD,
stale probability weights may lead to more variance but we will always get an unbiased estimator
of the true gradient, even when we get rid of all the synchronization barriers.

\begin{figure}[h]
%\minipage{0.95\textwidth}
    \includegraphics[scale=0.75]{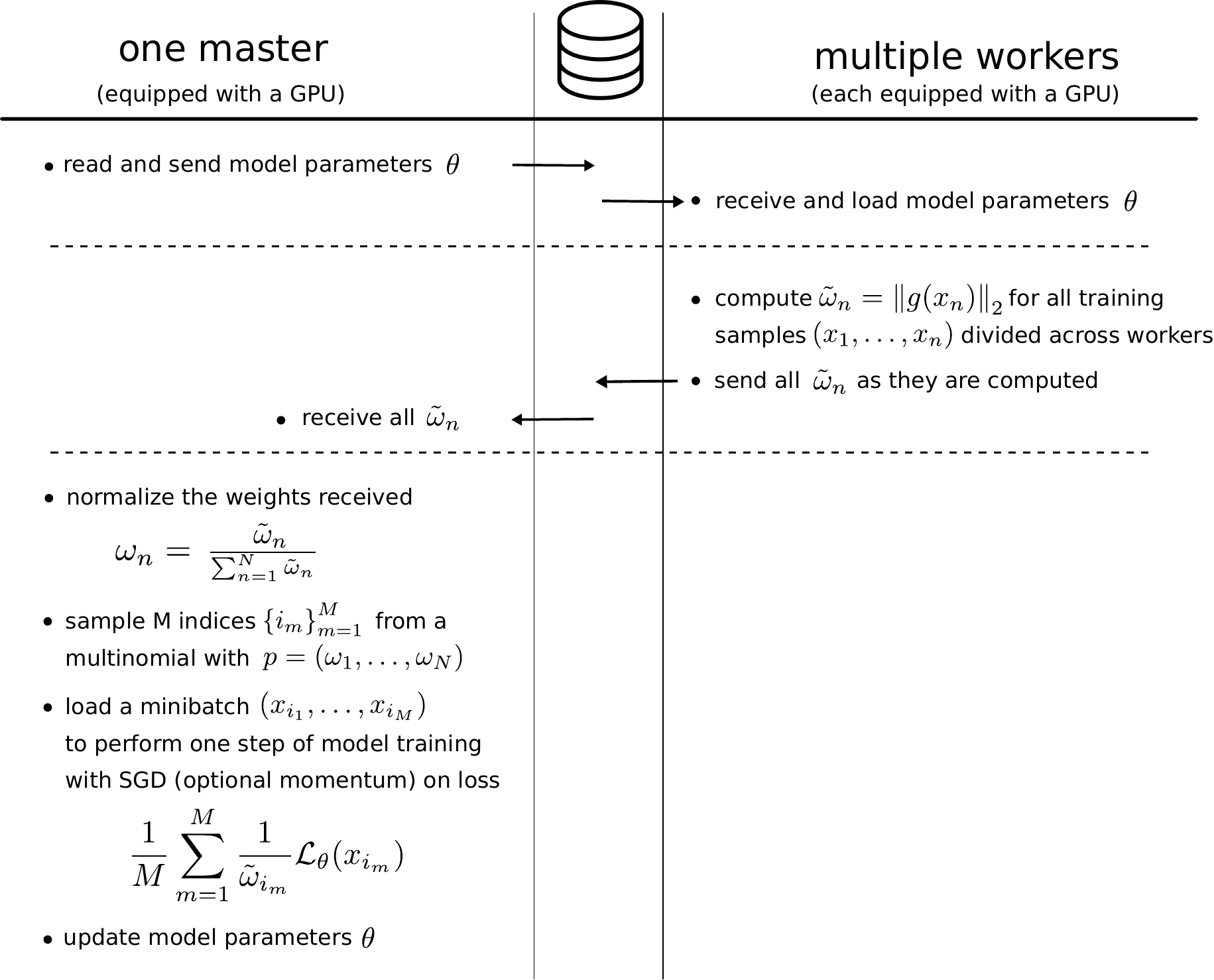}
    \caption{The actual distributed training experiment that we run relies on 3 kinds of actors.
             We have one \emph{master} process that is running ISSGD.
             We have one \emph{database} process in charge of storing and exchanging all kinds of measurements,
             as well as the parameters when they are communicated by the master to the workers.
             We have multiple \emph{worker} processes, each with one GPU, in charge of evaluating
             the quantities necessary for the master to do importance sampling.
             The master has to read the model parameters from the GPU before sending them
             to the database, and the workers also have to load them unto the GPU after receiving them.
             The horizontal dotted lines represent synchronization barriers that we can enforce to
             have an exact method, or that we can drop to have faster training in practice.
             }
    \label{fig:master-database-worker}
%\endminipage\hfill
\end{figure}

\vspace{0.5cm}

In the appendix we discuss three aspects of how training can be adapted to
be more practical and robust.
In section \ref{appsec:subset_weights} we
discuss the possibility of using only a subset of the probability weights,
filtering them based on how recently they have been updated.
In section \ref{appsec:gtrue_norm} we discuss how to approximate
$\left\Vert g^\textsc{true} \right\Vert_2^2$, which is not a quantity that
we absolutely need to compute for perform training, but which is something that
we like to monitor to assess the benefits of using ISSGD instead of regular SGD.
In section \ref{appsec:smoothing-constant} we add a smoothing constant to the
probability weights in order to make training more robust to sudden changes in
gradients.

\section{Experimental results}
\label{sec:experiments}

\subsection{Dataset and model}

We evaluated our model on the Street View House Numbers (SVHN) dataset from ~\citet{Netzer-wkshp-2011}.
We used the cropped version of the dataset (sometimes referred to as SVHN-2), which contains
about 600,000 32x32 RGB images of house number digits from Google Street View.

Since there is no standard validation set, we randomly split 5\% of the data to form our own validation set.
Since our per-example gradient norm computation (from section \ref{sec:dealing-minibatches})
does not work with parameter sharing models (such as RNNs and Convnets),
we consider the permutation invariant version of the SVHN task,
in which the model is forced to discard the spatial structure of the pixels.
While the permutation-invariant task is not practically relevant
(as the spatial structure of the pixels is useful),
it is commonly used as a testbed for studying fully connected neural networks
~\citep{Goodfeli-et-al-TR2013,Srivastava13}.

This is not meant to be a paper about exploring a variety of models,
and since we stick with the permutation-invariant task this already
limits our ability to use more interesting models. In any case, we
picked an MLP with 4 hidden layers, each with 2048 hidden units
and with a ReLU at its output (except for a softmax at the final layer).
We are very much aware that a convolutional model would perform better.

We have used Theano (\citep{bergstra+al:2010-scipy,Bastien-Theano-2012}) to
implement the model, and Redis as a database solution. The master and workers
are each equipped with a k20 GPU.

\subsection{Reduced training time and better prediction error}
\label{sec:main-experiment-with-plots}

We compare in figure \ref{fig:usgd-vs-isgd-train-loss-2} the training loss
for a model trained with ISSGD (in green) and regular SGD (in blue).
We used 3 workers to help with the master. In the case of regular SGD, we also used
a worker in the background to be able to compute statistics as we go along without
imposing that burden on the process training the model. To make sure that the results
are not due to the random initialization of parameters, we ran this experiment 50 times.
We report here the median (thicker line), and the quartiles 1 and 3 above and below (thinner lines).
This represents a ``tube'' into which half of the trajectories fit.
%The decision to use this instead of the mean along with variance error bars is because
%we are in a situation where we approach very closely a target of 0.0 for the training prediction error,
%and it's somewhat conceptually strange to have error bars that would go under the 0.0.
%Instead we represent this as a

In all the figures from this section, we always compare the same two sets of hyperparameters.
On the left we always have a setting where the learning rate is higher (0.01) and
where we smoothe the probability weights by adding a constant (+10.0) to them
(see section \ref{appsec:smoothing-constant} in the appendix for more explanations on this technique).
On the right we always have a setting where the learning rate is smaller (0.001)
and where the smoothing constant is also smaller (+1.0).

\begin{figure}[htb]
\center
\begin{subfigure}[h]{0.48\textwidth}
  \includegraphics[width=\linewidth]{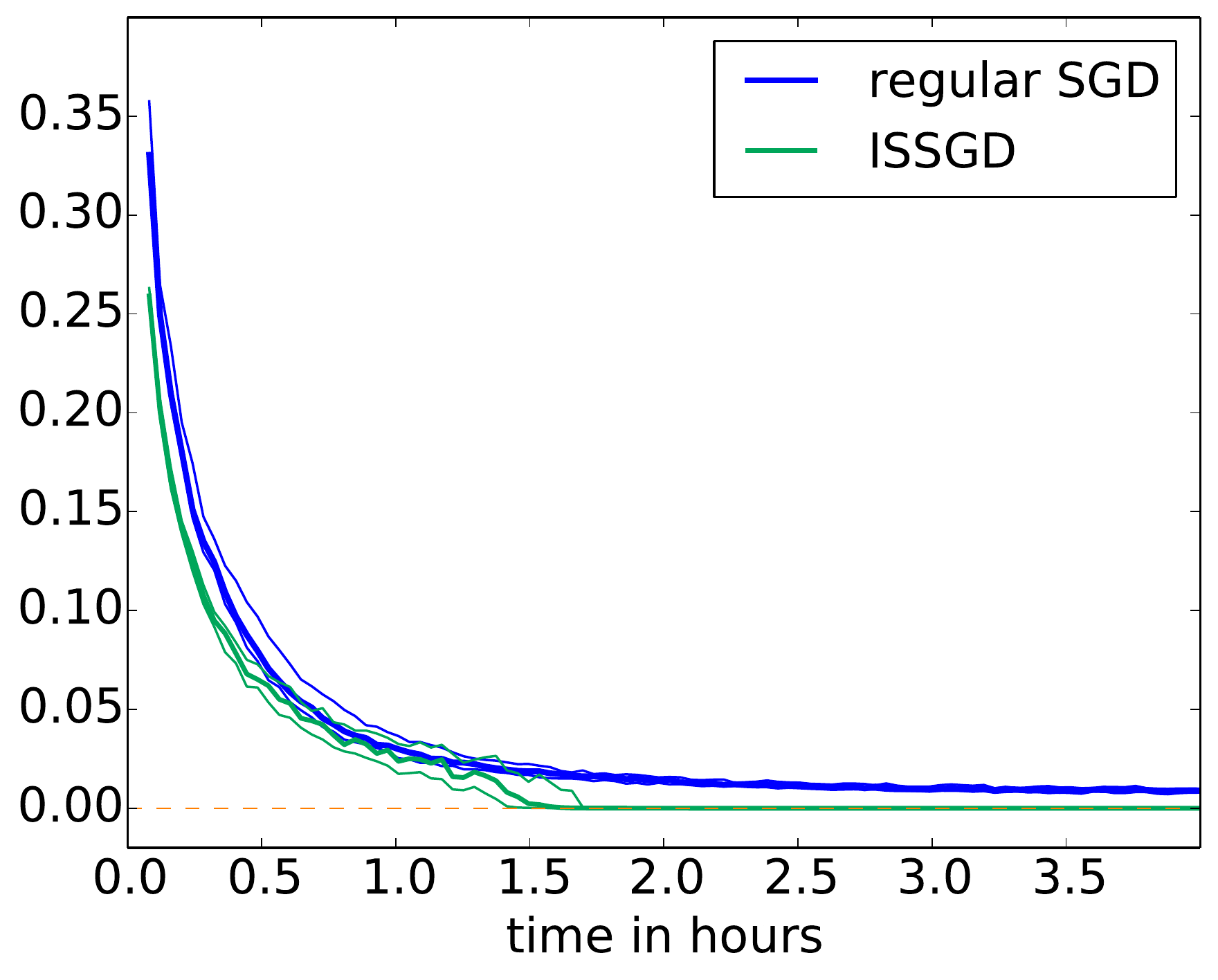}
  \caption{train loss, learning rate $0.01$}
  \label{fig:usgd-vs-isgd-train-loss-2-a}
\end{subfigure}
\hfill
\begin{subfigure}[h]{0.48\textwidth}
  \includegraphics[width=\linewidth]{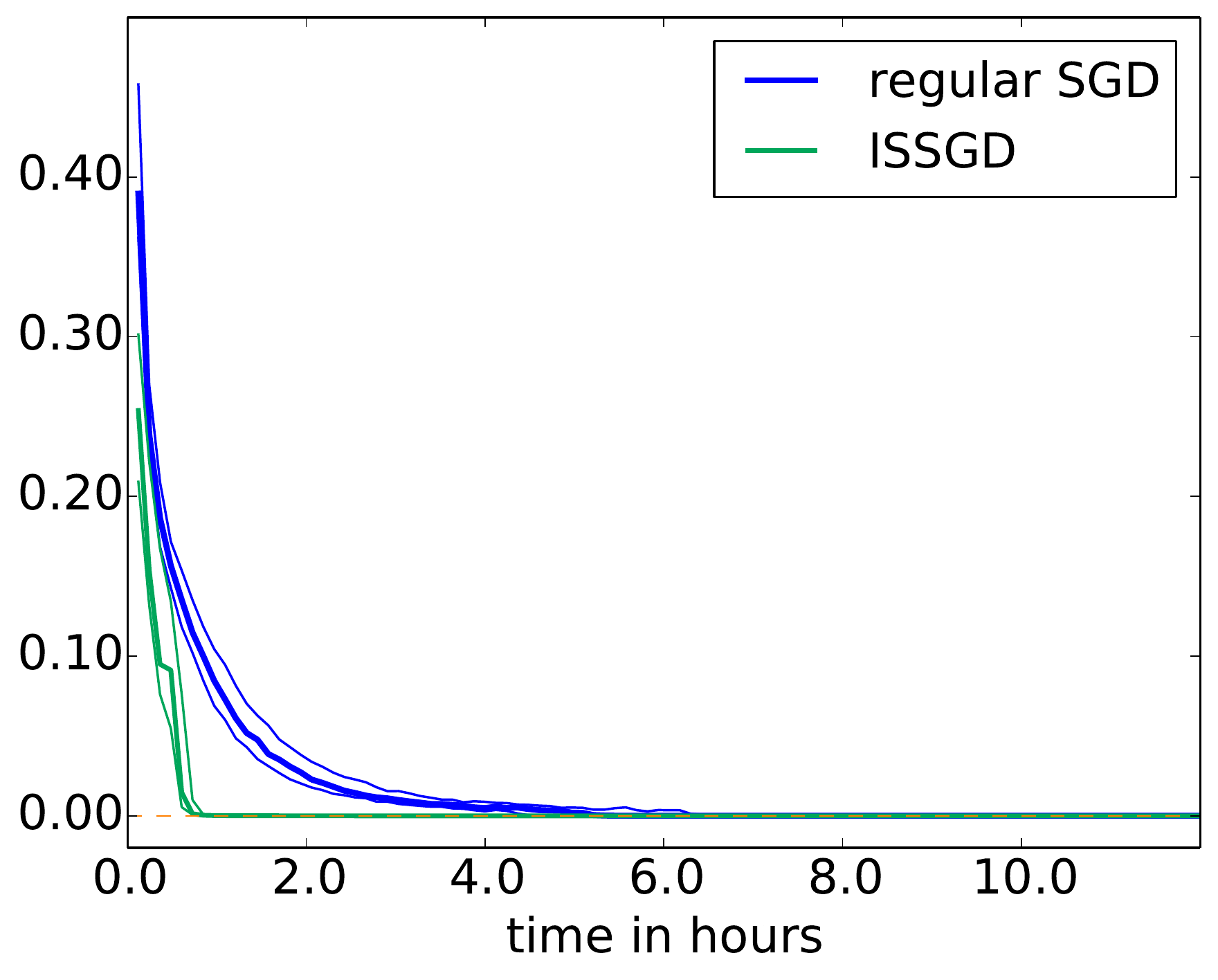}
  \caption{train loss, learning rate $0.001$}
  \label{fig:usgd-vs-isgd-train-loss-2-b}
\end{subfigure}
\begin{subfigure}[h]{0.48\textwidth}
  \includegraphics[width=\linewidth]{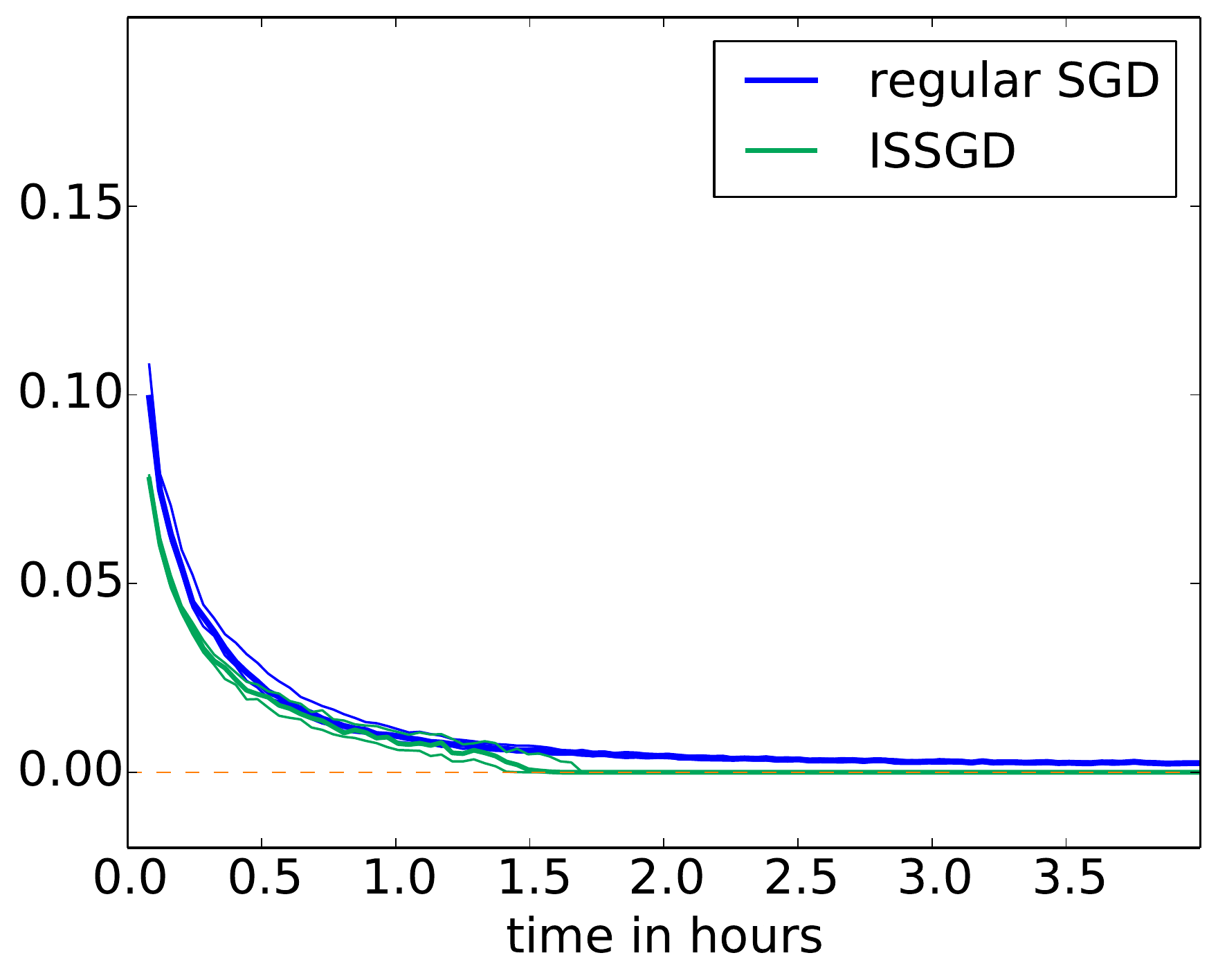}
  \caption{train prediction error, learning rate $0.01$}
  \label{fig:usgd-vs-isgd-train-loss-2-c}
\end{subfigure}
\hfill
\begin{subfigure}[h]{0.48\textwidth}
  \includegraphics[width=\linewidth]{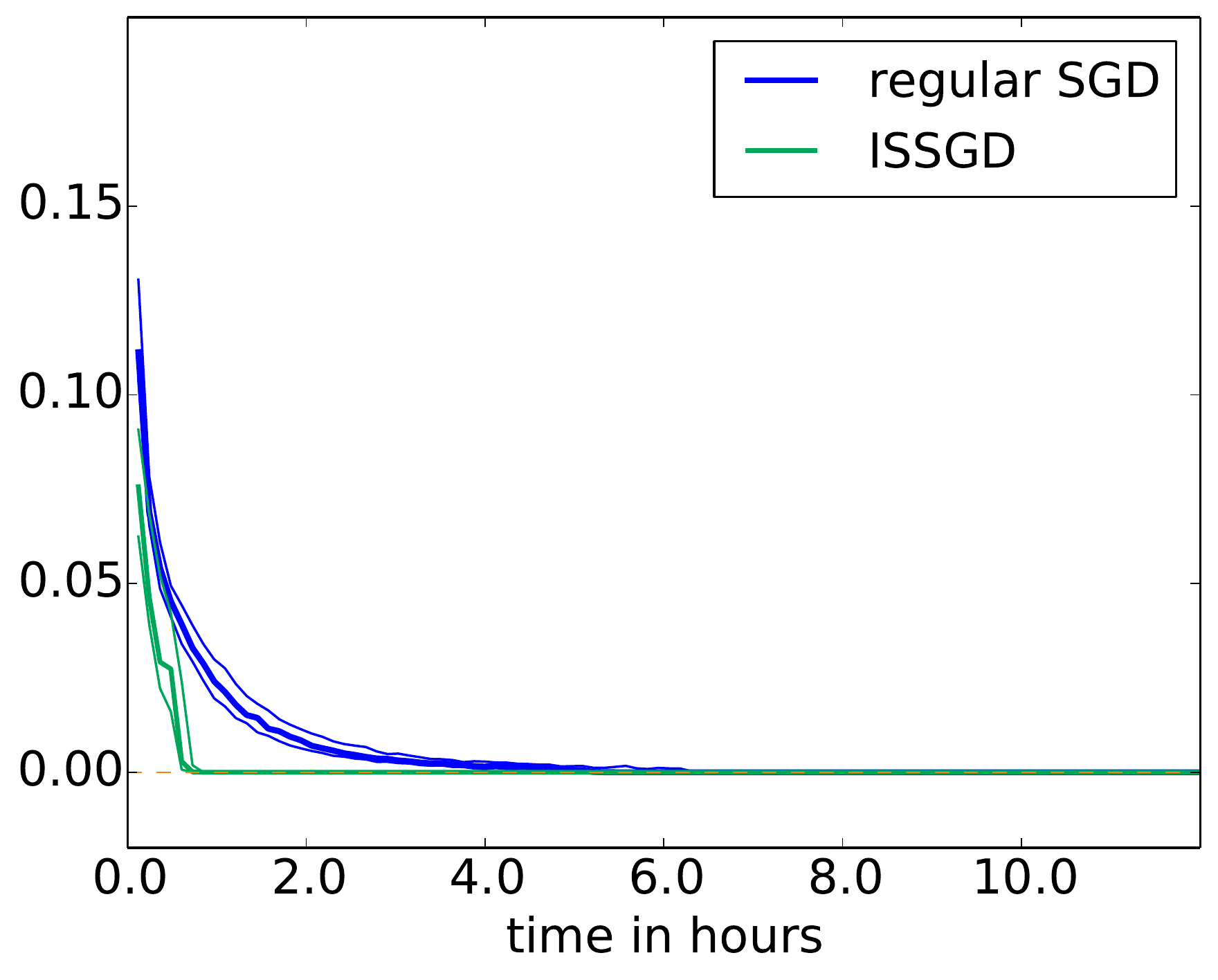}
  \caption{train prediction error, learning rate $0.001$}
  \label{fig:usgd-vs-isgd-train-loss-2-d}
\end{subfigure}
\caption{Here in the top two plots we compare the training loss optimized with two sets of hyperparameters.
        On the top-left we use a higher learning rate, but also a higher smoothing of the importance weights
        to stabilize the algorithm.
        In the two top plots, these are the actual quantities that are getting minimized by our procedure.
        We can see that, in both cases, ISSGD minimizes the loss more quickly than regular SGD,
        and it actually reaches 0.0. Those results are the median quantities reported
        during 50 runs for each set of hyperparamters, using a different random initialization.
        We also show the quartiles 1 and 3 in thinner lines to get an idea of the distributions.
        In the two bottom plots we also report the prediction error on the training set for each method.
        Note the different time scale between the left and the right.
        }
\label{fig:usgd-vs-isgd-train-loss-2}
\end{figure}

%\begin{figure}[!htb]
\begin{figure}[htb]
  %\centering
  \begin{subfigure}[h]{0.48\textwidth}
    \includegraphics[width=\linewidth]{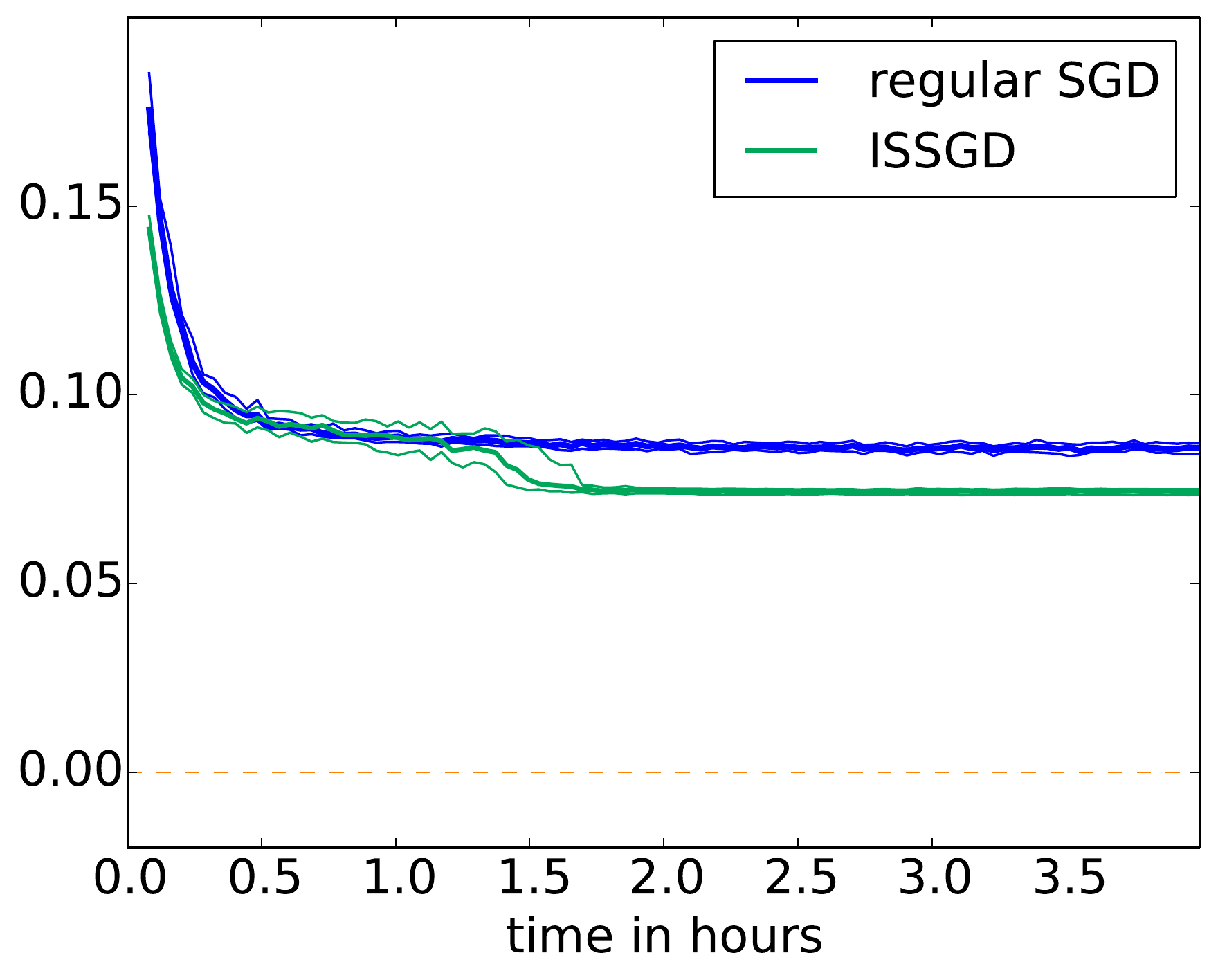}
    \caption{prediction error test, learning rate $0.01$}
    \label{fig:usgd-vs-isgd-prediction-error-test-3-a}
  \end{subfigure}
\hfill
\begin{subfigure}[h]{0.48\textwidth}
  \includegraphics[width=\linewidth]{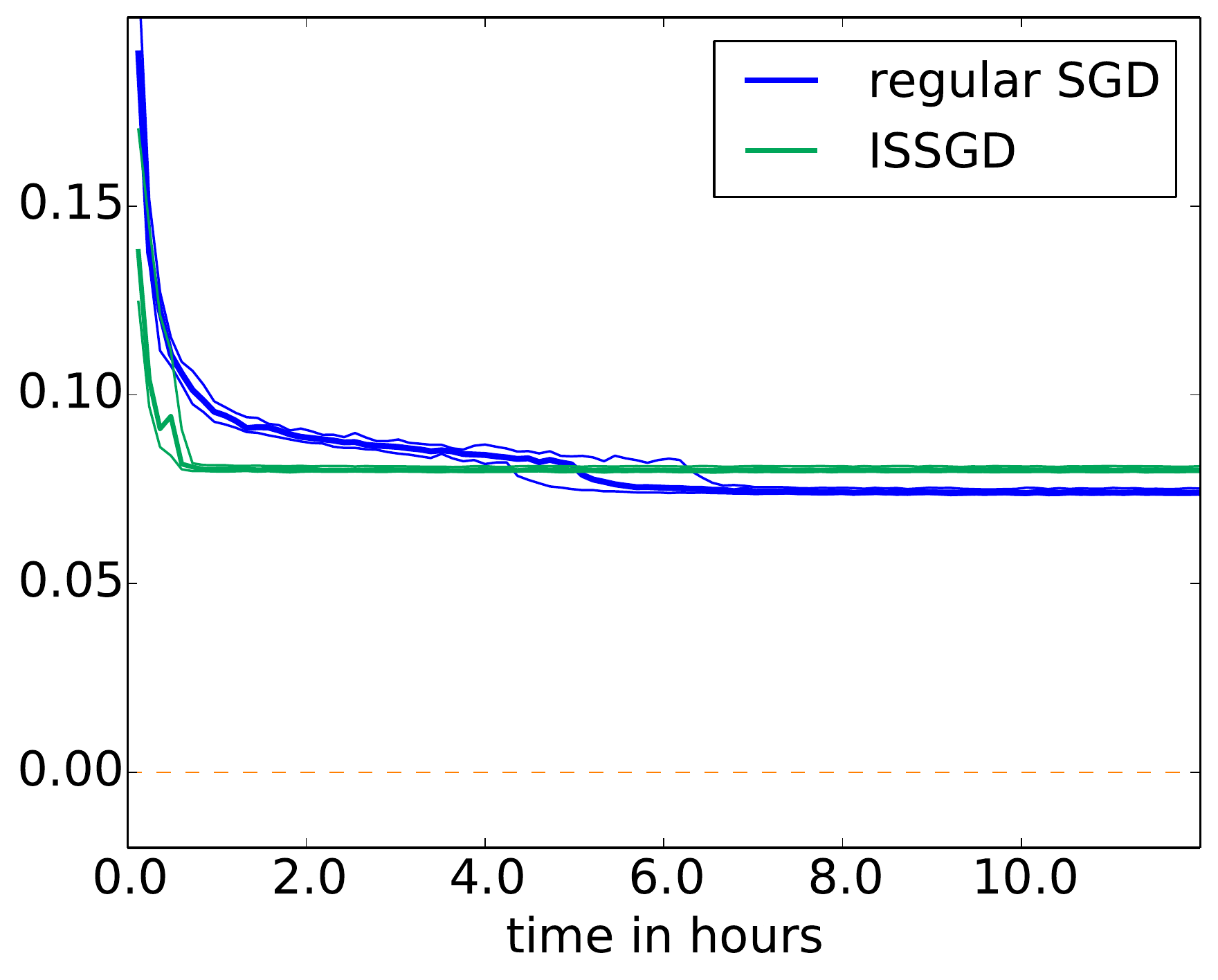}
  \caption{prediction error test, learning rate $0.001$}
  \label{fig:usgd-vs-isgd-prediction-error-test-3-b}
\end{subfigure}
\caption{Here we report the prediction error on the test set.
         Just like in figure \ref{fig:usgd-vs-isgd-train-loss-2}, we report the median results
         over 50 runs with the same two sets of hyperparameters.
         In a fairly consistent way, we have that one setup has a better generalization error
         for ISSGD (on the left plot), and the opposite happens in the other scenario (right plot).
         We believe that this can be explained by ISSGD converging quickly to a configuration
         that minimizes the loss perfectly, after which it just gives up trying to do better.
         Regular SGD, on plot (b), %\ref{fig:usgd-vs-isgd-prediction-error-test-3-b},
         would appear to experience some kind of regularization due to its variance, and it would continue
         to optimize over the course of 6 hours instead of only one hour
         (as shown on plot \ref{fig:usgd-vs-isgd-train-loss-2-b}).
         %Both training set error and generalization error are significantly improved.
         %The final values seen here are reported in Table \ref{table:test-errors-svhn}.
         %We also note that the error curves using importance sampling are somewhat noisier earlier in training.
         %Using a small learning rate makes these curves more stable, but it also makes training longer,
         %as one should expect.
         %We also show the validation loss to observe that overfitting occurs right before 0.5 hours of training.
         }
\label{fig:usgd-vs-isgd-prediction-error-test-3}
\end{figure}

\begin{table}[tb]
\begin{center}
 \begin{tabular}{||c c||}
 \hline
 Model & Test Error \\ [0.5ex]
 \hline\hline
 SGD ~\citep{davis2013lowrank} & 9.31 \\
 \hline
 SGD (ours) & 0.0754 \\
 \hline
 Importance Sampling SGD & \textbf{0.0756} \\ [1ex]
 \hline
\end{tabular}
\end{center}
\caption{Test Error on Permutation Invariant SVHN Datasets with different methods.
Our results are aggregated from 50 runs with random initialization, and we report
the average prediction error (as percentage) over the final 10\% iterations.
In this case, the measured results are similar when using early stopping on validation set.}
\label{table:test-errors-svhn}
\end{table}

In figure \ref{fig:usgd-vs-isgd-train-loss-2} we can see that in both cases
ISSGD minimizes the train loss more quickly than regular SGD, and it actually reaches 0.0.
This obviously corresponds to overfitting, but since we are presenting here an optimization
method, it seems natural to celebrate the fact that it can minimize the objective function
faster and better.

In figure \ref{fig:usgd-vs-isgd-prediction-error-test-3} we show the test prediction error.
These results are not so easy to interpret, and we see that faster convergence does not always
lead to a better generalization error. This suggests that regular SGD benefits here from
a kind of regularization effect.

We also report in table \ref{table:test-errors-svhn} what are the final prediction errors for
both methods (averaged over the last 10\% of the timesteps plotted).
%In Table \ref{table:test-errors-svhn} we compare the final test prediction errors for regular SGD and ISSGD.
We picked the set of hyperparameters that had the best validation prediction error and reported the
test prediction errors. Unsurprizingly, this corresponds to using
the result from figure \ref{fig:usgd-vs-isgd-prediction-error-test-3-a} for ISSGD and
figure \ref{fig:usgd-vs-isgd-prediction-error-test-3-b} for regular SGD.
The final values are very similar for the two methods.

\subsection{Variance reduction}

Here we look at the values of values of $\Tr(\Sigma(q))$
during the ISSGD training from the previous section
(which led to figure \ref{fig:usgd-vs-isgd-train-loss-2} and figure \ref{fig:usgd-vs-isgd-prediction-error-test-3}).

We would like to compare the values of $\Tr(\Sigma(q))$ for
$(q_\textsc{ideal}, q_\textsc{stale}, q_\textsc{unif})$. Note that $q_\textsc{unif}$ does
not mean here that we trained with the regular SGD (that assigns the same probability to each training example).
It means that, during ISSGD training, we can report the value of $\Tr(\Sigma(q))$
that we \emph{would} get if we performed the next step with regular SGD.
In figure \ref{fig:sqrtcov}, we refer to this as ``SGD, ideal''.
We compare it to ``ISSGD, ideal'', which corresponds to the
best possible situation for our method, $\Tr(\Sigma(q_\textsc{ideal}))$,
which is not necessarily achieved in practice.

In section \ref{appsec:smoothing-constant} of the Appendix we describe how we
add a constant to the probability weights in order to make the method more robust.
We are trading away potential gains to make training more stable.

% With an aggressive
%learning rate such as the one used in the experiments from figure \ref{fig:usgd-vs-isgd-4},
%we had to resort to adding +10.0 to all the importance weights. This brings our method
%closer to regular SGD than we would like, but we are showing this version here nevertheless.
%We refer to that as ``ISSGD, actual'' in figure \ref{fig:sqrtcov}.
%This quantity also takes into account the staleness factor, so it really is the
%thing that we're dealing with when training.

On both plots of figure \ref{fig:sqrtcov} we show the ``ideal'' measurements that
we would get with exact probability weights,
and we compare with the ``stale'' measurements that we get with probability weights used in the actual experiments,
which are all stale to varying degrees. On those stale curves, we show the effects
of using the actual additive constant to the probability weights, and the effects
of using an alternate one.
Bear in mind that, in both cases, the validation loss reached its minimum in around 30 minutes,
and these plots are shown for 2.5 hours. Also, these are the $\Tr(\Sigma(q))$ with respect to the
gradient on the training set. One naturally expects that gradient to converge to 0.0 during
the overfitting regime.

\begin{figure}[!htb]
\center
\begin{subfigure}[h]{0.45\textwidth}
  \includegraphics[width=\linewidth]{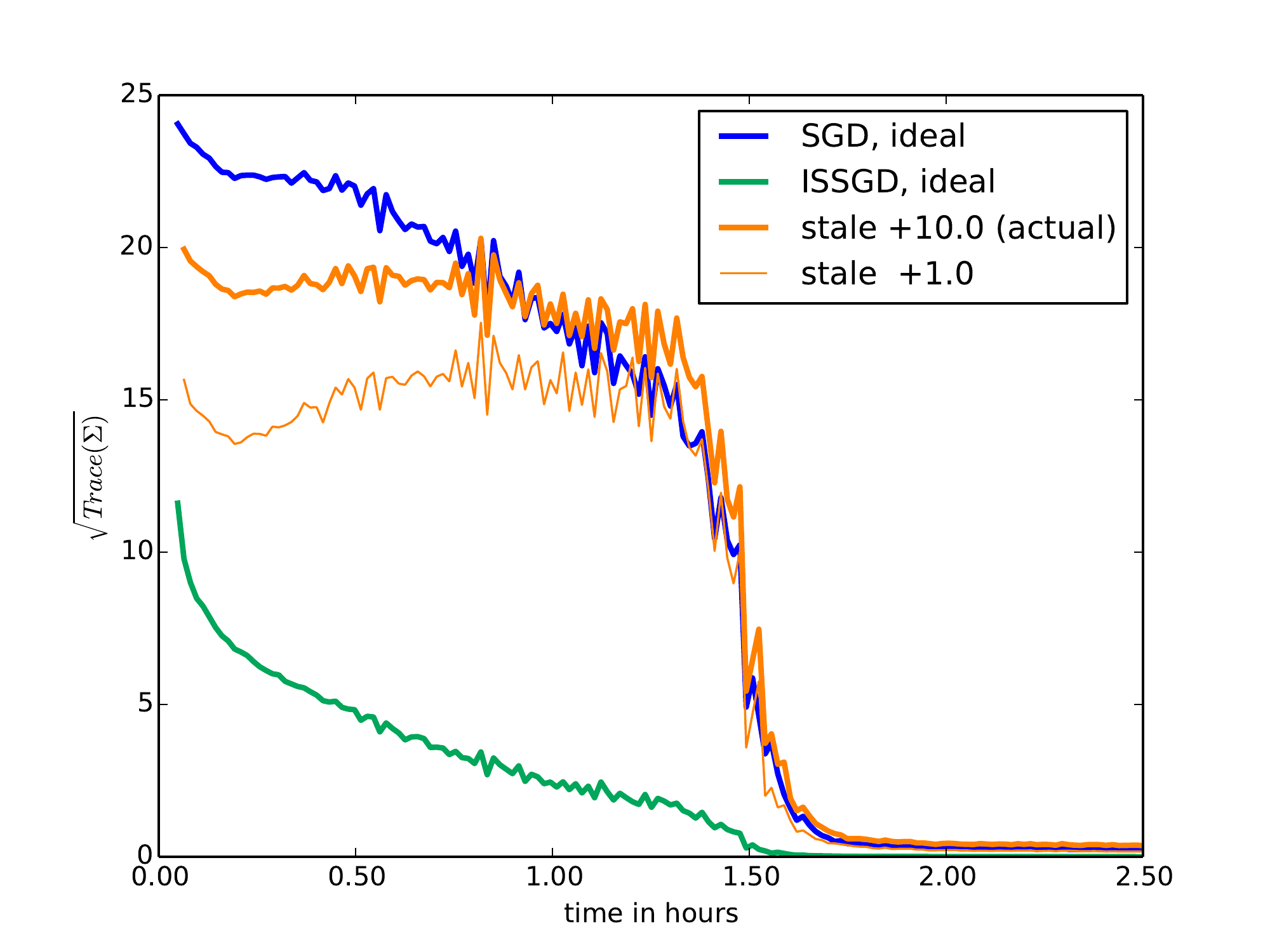}
  \caption{learning rate $0.01$, add smoothing $+10.0$}
  \label{fig:sqrtcov-2}
\end{subfigure}
\hfill
\begin{subfigure}[h]{0.45\textwidth}
  \includegraphics[width=\linewidth]{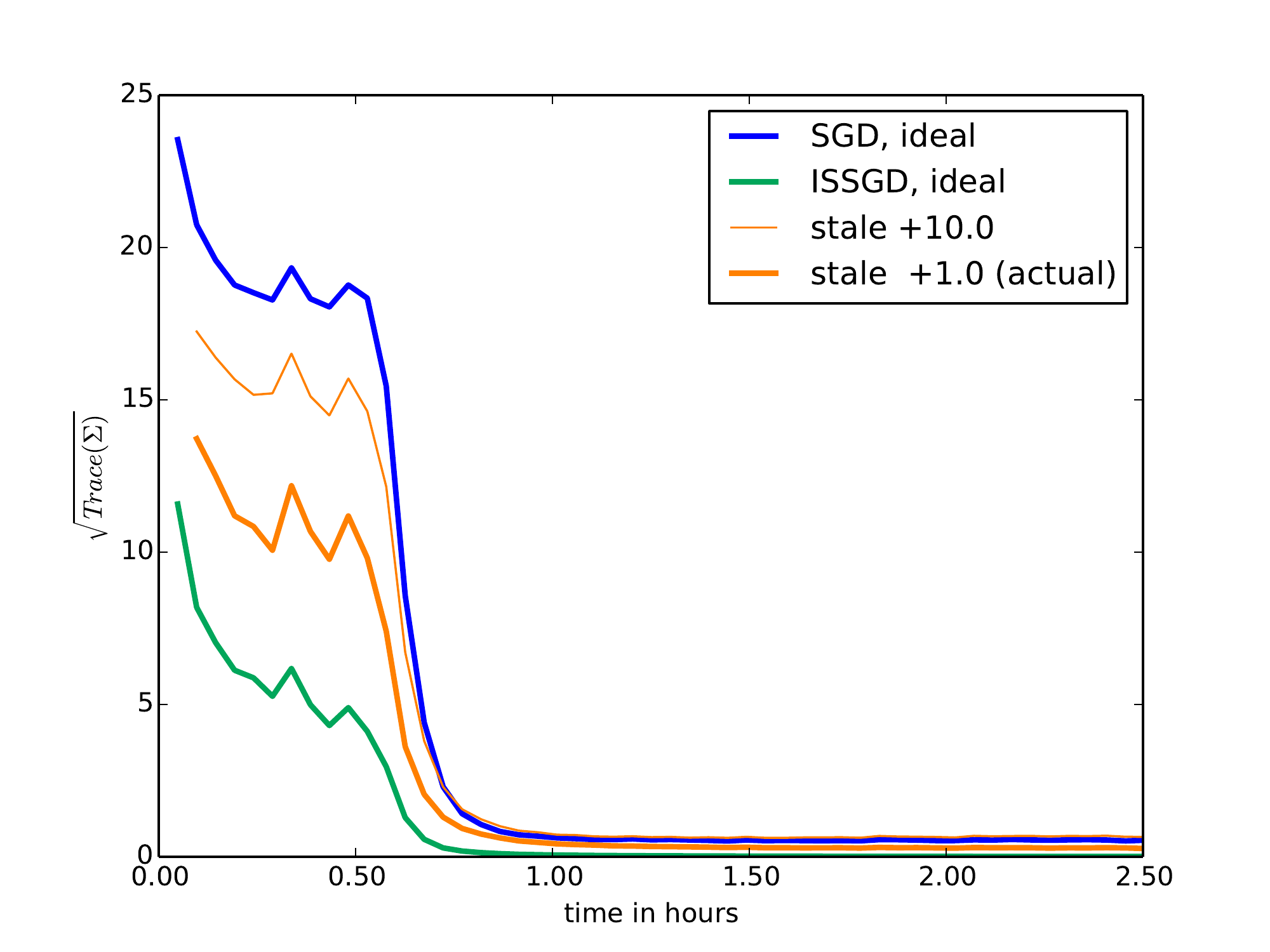}
  \caption{learning rate $0.001$, add smoothing $+1.0$}
  \label{fig:sqrtcov-3}
\end{subfigure}
\caption{Square root of trace of covariance for different proposals $q$.
         We show here the median results aggregated over 50 runs of ISSGD.
         These plots come from the same hyperparameters used for figure \ref{fig:usgd-vs-isgd-train-loss-2}.
         On the left plot, we use a higher learning rate in the hopes of making convergence faster.
         This required the probability weights to be smoothed by adding a constant (+10.0)
         to all the probability weights, and this washed away a part of the variance-reduction
         benefits of using ISSGD.
         On the right plot, we used a smaller learning rate, and we still got comparably fast convergence.
         However, because of the additive constant +1.0 used, these runs were closer to the
         ideal ISSGD setting.
         The point of these plots is to show that with ISSGD we can a smaller measurement of $\Tr(\Sigma(q))$.
         This happens clearly on the right plot, but not as convincingly on the left.         }
\label{fig:sqrtcov}
\end{figure}

Note that in figure \ref{fig:sqrtcov} we report the square root of those values
in order to have it be on the same scale and the gradients themselves
(this is analogous to reporting $\sigma$ instead of $\sigma^2$).

\section{Future work}

One of the constraints that we are facing is that proposition \ref{prop:individual-grad-norms}
works with models with only fully-connected layers. This rules out all the convolutional neural networks,
which are very popular and very useful.
% We have worked on a similar result for
%convolutions, but it has eluded us so far. If one were to be found, we could try out
%ISSGD on many of the popular models.

One alternative would be to use an approximate formula for the individual gradient norms
for convolutional layers. Either something naive (such as applying proposition \ref{prop:individual-grad-norms}
without proper justification), or possibly even ignoring the contributions from those layers.
This would yield an importance sampling scheme that would be of lesser quality,
but it would also be hard to evaluate how much we actually suffer for that.
%The equivalent of figure \ref{fig:sqrtcov} would be expensive to compute because
%we could not use the same trick to handle minibatches.

We have avoided direct comparisons with ASGD in this paper because we are not
currently in possession of a good production-quality ASGD implementation.
We would certainly like to see how ASGD could be combined with ISSGD,
whether this would create positive interactions or whether the two methods
would impede each other.

Note that there are alternative ways to combine our method with ASGD,
and they are not equally promising.
Our recommendation would be to get rid of the master/workers
distinction and have only workers (or ``peers'') along with a parameter server
(or shared memory, or whatever synchronization method is used to aggregate the gradients and parameters).
Whenever a gradient contribution is computed, the importance weights can
be obtained at the same time. These can be shared in the same way that the gradients
are shared, so that all the workers are able to use the importance weights
to run ISSGD steps.

%The question of staleness of the probability weights is also interesting.
%We have used a certain cutoff level to use only a certain proportion of the
%probability weights that are the most recent. However, apart from running
%the algorithm and looking at the results, it is hard to determine what the
%best approach should be. We are not in a position to formulate good heuristics
%besides that. For example, one might expect that, close to the end of the training,
%when the parameters stay more or less constant, we could relax the constraints
%on the staleness of the weights.

%The question of the additive smoothing constant is also one that is somewhat
%unresolved. Right now we have found some values that work well, but no theory
%to explain what the users should prefer, or should aim to get. If we could make
%such prescriptions, we might be able to use an adaptive method to tweak this
%smoothing constant during training.

\section{Conclusion}

We have introduced a novel method for distributing neural network training by
using multiple machines to search for the most informative examples to train on.
This method led to significant improvements in training time on permutation invariant SVHN.
Our results demonstrated that importance sampling reduced the variance of the gradient estimate,
even in the distributed setting where the importance weights are not exact.
One area for future work is extending this method to models that use parameter sharing (such as convnets and RNNs),
either by finding a new formula for per-example gradient norms or by finding
an approximation to the gradient norm that is easy to compute.
%In our experiments we showed that our method is easily fooled by examples which are highly noisy.
%Learning to distinguish between examples that are learnable but difficult and
%examples where the features contain no information about the target could be an interesting,
%though challenging, avenue for future research.
Finally, much of the most successful work on data parallel distributed deep learning has used a variant of Asynchronous SGD.
It would be useful to understand exactly how our method compares with
Asynchronous SGD and to see if further improvement is gained by using both approaches simultaneously.

\subsubsection*{Acknowledgments}

The authors would like to acknowledge the support of the following agencies for
research funding and computing support: NSERC, Calcul Qu\'{e}bec,
Compute Canada, the Canada Research Chairs and CIFAR, and the Nuance Foundation.

The authors would like to acknowledge the stimulating discussions with
Ian Goodfellow, Zack Lipton, Kari Torkkola, Daniel Abolafia, and Ethan Holly.

We would also like to thank the developers of Theano.
\footnote{http://deeplearning.net/software/theano/}

%\section{Extra section for notes during editing process}
%
%\editing{
%Distributed ASGD with GPUs.  Results seem bad.  Not clear if paper was published anywhere.
%Basically shows that ASGD with 2-8 GPUs accelerates training over the first few epochs, but also leads to oscillations in the training error.}
%\begin{verbatim} http://arxiv.org/pdf/1312.6186v1.pdf \end{verbatim}
%
%\editing{Distributed ASGD for speech.  Didn't show any training curves, but the results suggest that ASGD trains faster and gets to a lower final error rate.}
%\begin{verbatim} http://ieeexplore.ieee.org/stamp/stamp.jsp?tp=&arnumber=6638950 \end{verbatim}

%\subsubsection*{References}

%\newpage

\bibliography{aigaion,ml}
\bibliographystyle{iclr2015}

\appendix

\section{Importance sampling in theory}

\setcounter{theorem}{0}
\setcounter{corollary}{0}
\setcounter{proposition}{0}

\subsection{Extending beyond a single dimension}
\label{appsec:extending-importance-sampling}

\begin{theorem}
%\label{th:multi-dim-imp-samp} % no label in the appendix

Let $\mathcal{X}$ be a random variable in $\mathbb{R}^{d_1}$ and $f(x)$ be any function
from $\mathbb{R}^{d_1}$ to $\mathbb{R}^{d_2}$.
Let $p(x)$ be the probability density function of $\mathcal{X}$, and let $q(x)$ be a valid
proposal distribution for importance sampling with the goal of estimating
\begin{equation}
\mathbb{E}_{p} \left[ f(x) \right] = \int p(x) f(x) dx = \mathbb{E}_{q}\left[ \frac{p(x)}{q(x)} f(x) \right].
\end{equation}

The context requires that $q(x) > 0$ whenever $p(x) > 0$. We know that the
importance sampling estimator
\begin{equation}
\frac{p(x)}{q(x)} f(x) \hspace{1em} \textrm{with $x \sim q$}
\end{equation}
has mean $\mu = \mathbb{E}_{p}\left[ f(x) \right]$ so it is unbiased.

Let $\Sigma(q)$ be the covariance of that estimator, where we include $q$
in the notation to be explicit about the fact that it depends on the choice of
$q$.

Then the trace of $\Sigma(q)$ is minimized by the following optimal proposal $q^*$ :
\begin{equation}
q^*(x) = \frac{1}{Z} p(x) \left\Vert f(x) \right\Vert_2 \hspace{0.5em} \textrm{where} \hspace{0.5em}
      Z = \int p(x) \left\Vert f(x) \right\Vert_2 dx
\end{equation}
which achieves the optimal value
\[
\Tr(\Sigma(q^*)) = \left(\mathbb{E}_{p}\left[\left\Vert f(x)\right\Vert _{2}\right]\right)^{2} - \left\Vert \mu \right\Vert_2^2.
\]
\end{theorem}

\begin{proof}

This proof is almost exactly the same as the well-known result in one dimension
involving Jensen's inequality.
Everything follows from the decision to minimize $Tr\left(\Sigma\right)$
and the choice of $q^*$. Nevertheless, we include it here so the reader can
get a feeling for where $q^*$ comes into play.

When sampling from $q(x)$ instead of $p(x)$, we are looking at how
the unbiased estimator
\[
\mathbb{E}_{q(x)}\left[\frac{p(x)}{q(x)}f(x)\right]
\]
which has mean $\mu$ and covariance $\Sigma(q)$. We make use of the fact
that the trace is a linear function, and that $\Tr(\mu \mu^T) = \left\Vert \mu\right\Vert _{2}^{2}$.
The trace of the covariance is given by
\begin{eqnarray}
\Tr\left(\Sigma(q)\right) & = & \Tr \left( \mathbb{E}_{q(x)}\left[\left(\frac{p(x)}{q(x)}f(x)-\mu\right)\left(\frac{p(x)}{q(x)}f(x)-\mu\right)^{T}\right] \right) \nonumber \\
 & = & \Tr\left(\mathbb{E}_{q(x)}\left[\left(\frac{p(x)}{q(x)}f(x)\right)\left(\frac{p(x)}{q(x)}f(x)\right)^{T}\right]-\mu\mu^{T}\right) \nonumber \\
 & = & \mathbb{E}_{q(x)}\left[Tr\left(\left(\frac{p(x)}{q(x)}f(x)\right)\left(\frac{p(x)}{q(x)}f(x)\right)^{T}\right)\right]-\left\Vert \mu\right\Vert _{2}^{2} \nonumber \\
 & = & \mathbb{E}_{q(x)}\left[\left\Vert \frac{p(x)}{q(x)}f(x)\right\Vert _{2}^{2}\right]-\left\Vert \mu\right\Vert _{2}^{2}. \label{eqn:isgd-trcov-anyq}
\end{eqnarray}

There is nothing to do about the $\left\Vert \mu\right\Vert _{2}^{2}$
term since it does not depend on the proposal $q(x)$. Using Jensen's
inequality, we get that
\[
\mathbb{E}_{q(x)}\left[\left\Vert \frac{p(x)}{q(x)}f(x)\right\Vert _{2}^{2}\right]\geq\mathbb{E}_{q(x)}\left[\left\Vert \frac{p(x)}{q(x)}f(x)\right\Vert _{2}\right]^{2}=\left(\int q(x)\frac{p(x)}{q(x)}\left\Vert f(x)\right\Vert _{2}dx\right)^{2}=\left(\mathbb{E}_{p(x)}\left[\left\Vert f(x)\right\Vert _{2}\right]\right)^{2}.
\]
This means that, for any proposal $q(x)$, we cannot do better than
$\left(\mathbb{E}_{p(x)}\left[\left\Vert f(x)\right\Vert _{2}\right]\right)^{2} - \left\Vert \mu\right\Vert _{2}^{2}$.
All that is left is to take the proposal $q^{*}$ in the statement of the theorem,
to evaluate $Tr(\Sigma(q^*))$ and to show that it matches that lower bound.

We have that
\begin{eqnarray}
\Tr(\Sigma(q^*)) & = & \mathbb{E}_{q^{*}(x)}\left[\left\Vert \frac{p(x)}{q^{*}(x)}f(x)\right\Vert _{2}^{2}\right] - \left\Vert \mu\right\Vert _{2}^{2} \nonumber \\
                & = & \int q^{*}(x)\left(\frac{p(x)}{q^{*}(x)}\right)^{2}\left\Vert f(x)\right\Vert _{2}^{2}dx - \left\Vert \mu\right\Vert _{2}^{2} \nonumber \\
                & = & \int\frac{p(x)^{2}}{q^{*}(x)}\left\Vert f(x)\right\Vert _{2}^{2}dx - \left\Vert \mu\right\Vert _{2}^{2} \label{eqn:last-general-q} \\
                & = & \int\frac{p(x)^{2} Z}{p(x) \left\Vert f(x)\right\Vert_2 }\left\Vert f(x)\right\Vert _{2}^{2}dx - \left\Vert \mu\right\Vert _{2}^{2} \hspace{0.5em} \textrm{where} \hspace{0.5em}
                      Z = \int p(x) \left\Vert f(x) \right\Vert_2 dx \nonumber \\
                & = & \left(\mathbb{E}_{p(x)}\left[\left\Vert f(x)\right\Vert _{2}\right]\right)^{2} - \left\Vert \mu\right\Vert _{2}^{2} \nonumber
\end{eqnarray}
which is the minimal value achievable, so $q^*$ is indeed the best proposal in
terms of minimizing $Tr(\Sigma(q))$.
\end{proof}
Note also that the single-dimension equivalent, mentioned in section \ref{sec:classic-case-single-dimension},
is a direct corollary of this proposition, because the Euclidean norm turns
into the absolute value.

\vspace{1em}

\begin{corollary}
%\label{co:trace-covariance-imp} % no labels in appendix
Using the context of importance sampling as described in theorem \ref{th:multi-dim-imp-samp}, let $q(x)$
be a proposal distribution that is proportional to $p(x)h(x)$ for some
function $h:\mathcal{X}\rightarrow\mathbb{R}^+$.
As always, we require that $h(x) > 0$ whenever $f(x) > 0$.

Then we have that the trace of the covariance of the importance sampling estimator is given by
\[
\Tr(\Sigma(q)) = \left( \int p(x) h(x) dx \right) \left( \int p(x) \frac{ \left\Vert f(x) \right\Vert_2^2 }{h(x)} dx \right) - \left\Vert \mu \right\Vert_2^2,
\]
where $\mu = \mathbb{E}_{p(x)}\left[ f(x) \right]$.
Moreover, if $p(x)$ is not known directly, but we have access to a dataset $\mathcal{D} = \{x_n\}_{n=1}^{\infty}$ of samples
drawn from $p(x)$, then we can still define $q(x) \propto p(x) h(x) $ by associating the probability
weight $\tilde{\omega}_n = h(x_n)$ to every $x_n \in \mathcal{D}$.

To sample from $q(x)$ we just normalize the probability weights
\[
\omega_n = \frac{\tilde{\omega}_n}{ \sum_{n=1}^N \tilde{\omega}_n }
\]
and we sample from a multinomial distribution with argument $(\omega_1,\ldots,\omega_N)$ to
pick the corresponding element in $\mathcal{D}$.

In that case, we have that
\begin{eqnarray*}
\Tr(\Sigma(q)) & = & \left( \frac{1}{N} \sum_{n=1}^N \tilde{\omega}_n \right) \left( \frac{1}{N} \sum_{n=1}^N \frac{ \left\Vert f(x_n) \right\Vert_2^2 }{\tilde{\omega}_n} \right) - \left\Vert \mu \right\Vert_2^2 \\
               & = & \left( \frac{1}{N} \sum_{n=1}^N \omega_n \right) \left( \frac{1}{N} \sum_{n=1}^N \frac{ \left\Vert f(x_n) \right\Vert_2^2 }{\omega_n} \right) - \left\Vert \mu \right\Vert_2^2.
\end{eqnarray*}

%\begin{eqnarray}
%\Tr(\Sigma(q)) & = & \left( \frac{1}{N} \sum_{n=1}^N \tilde{\omega}_n \right) \left( \frac{1}{N} \sum_{n=1}^N \frac{ \left\Vert f(x_n) \right\Vert_2^2 }{\tilde{\omega}_n} \right) \\
%               & = & \left( \frac{1}{N} \sum_{n=1}^N \omega_n \right) \left( \frac{1}{N} \sum_{n=1}^N \frac{ \left\Vert f(x_n) \right\Vert_2^2 }{\omega_n} \right).
%\end{eqnarray}

\end{corollary}

\begin{proof}
We start from equation (\ref{eqn:isgd-trcov-anyq}), which applies to a general proposal $q$.
In fact, we make it to equation (\ref{eqn:last-general-q}) still without making assumptions on $q$.
At that point we can look at the normalizing constant of $q$, which is equal to
\[
Z = \int p(x) h(x) dx = \frac{1}{N} \sum_{n=1}^N h(x_n) = \frac{1}{N} \sum_{n=1}^N \tilde{\omega}_n.
\]
Then we have that
\begin{eqnarray}
\Tr(\Sigma(q)) & = & \int\frac{p(x)^{2} Z}{p(x) h(x) }\left\Vert f(x)\right\Vert _{2}^{2}dx - \left\Vert \mu\right\Vert _{2}^{2} \hspace{0.5em} \textrm{where} \hspace{0.5em}
      Z = \frac{1}{N} \sum_{n=1}^N \tilde{\omega}_n \\
               & = & \left( Z \right) \left( \int p(x) \frac{\left\Vert f(x)\right\Vert _{2}^{2}}{ h(x) } dx \right) \\
               & = & \left( \frac{1}{N} \sum_{n=1}^N \tilde{\omega}_n \right) \left( \frac{1}{N} \sum_{n=1}^N \frac{\left\Vert f(x_n)\right\Vert _{2}^{2}}{ \tilde{\omega}_n } \right)
\end{eqnarray}
The equivalent formula for $\omega_n$ instead of $\tilde{\omega}_n$ follows from
dividing the left expression by $\sum_{n=1}^N \tilde{\omega}_n$ and multiplying
the expression on the right by that constant.
\end{proof}

\subsection{Dealing with minibatches}
\label{appsec:dealing-minibatches}

\begin{proposition}
\label{prop:individual-grad-norms}
Consider a multi-layer perceptron (MLP) applied to minibatches of size $N$,
and with loss $\mathcal{L} = \mathcal{L}_1 + \ldots + \mathcal{L}_N$, where
$\mathcal{L}_n$ represents the loss contribution from element $n$ of the minibatch.

%\vspace{1em}

Let $(W,b)$ be the weights and biases at any particular fully-connected layer so that
$X W + b = Y$, where $X$ are the inputs to that layer and $Y$ are the outputs.

%\vspace{1em}

The gradients with respect to the parameters are given by
\begin{eqnarray*}
\frac{\partial \mathcal{L}}{\partial W} & = & \frac{\partial \mathcal{L}_1}{\partial W} + \ldots + \frac{\partial \mathcal{L}_N}{\partial W} \\
\frac{\partial \mathcal{L}}{\partial b} & = & \frac{\partial \mathcal{L}_1}{\partial b} + \ldots + \frac{\partial \mathcal{L}_N}{\partial b}
\end{eqnarray*}
where the values $\left( \frac{\partial \mathcal{L}_n}{\partial W}, \frac{\partial \mathcal{L}_n}{\partial b} \right)$
refer to the particular contributions coming from element $n$ of the minibatch.
\vspace{1em}
Then we have that
\begin{eqnarray*}
%\left\Vert \frac{\partial \mathcal{L}_n}{\partial W} \right\Vert_\textrm{F}^2 & = & \left\Vert X[n,:] \vspace{1em} \right\Vert_2^2 \hspace{1em} \left\Vert \frac{\partial \mathcal{L}}{\partial Y}[n,:] \right\Vert_2^2 \\
\left\Vert \frac{\partial \mathcal{L}_n}{\partial W} \right\Vert_\textrm{F}^2 & = & \Vert X[n,:] \Vert_2^2  \hspace{0.5em} \cdot \hspace{0.5em} \left\Vert \frac{\partial \mathcal{L}}{\partial Y}[n,:] \right\Vert_2^2 \\
\left\Vert \frac{\partial \mathcal{L}_n}{\partial W} \right\Vert_2^2 & = & \left\Vert \frac{\partial \mathcal{L}}{\partial Y}[n,:] \right\Vert_2^2, \\
\end{eqnarray*}
where the notation $X[n,:]$ refers to row $n$ of $X$, and similarly for $\frac{\partial \mathcal{L}}{\partial Y}[n,:]$.

That is, we have a compact formula for the Euclidean norms of the gradients of the parameters, evaluated for each $N$ elements of the minibatch independently.
\end{proposition}

\begin{proof}
The usual backpropagation rules give us that
\[
\frac{\partial \mathcal{L}}{\partial W} = X^T \frac{\partial \mathcal{L}}{\partial Y}
\hspace{1em} \textrm{and} \hspace{1em}
\frac{\partial \mathcal{L}}{\partial b} = \sum_{n=1}^N \frac{\partial \mathcal{L}}{\partial Y}[n,:].
\]
All the backpropagation rules can be inferred by analyzing the
shapes of the quantities involved and noticing that only one
answer can make sense.
If we focus on $\mathcal{L}_n$ for some $n\in\{1,\ldots,N\}$, then we can see that
\begin{equation}
\label{eqn:individual-grads-Wb}
\frac{\partial \mathcal{L}_n}{\partial W} = X[n,:]^T \frac{\partial \mathcal{L}}{\partial Y}[n,:]
\hspace{1em} \textrm{and} \hspace{1em}
\frac{\partial \mathcal{L}}{\partial b} = \frac{\partial \mathcal{L}}{\partial Y}[n,:].
\end{equation}
Note here that $X[n,:]^T \frac{\partial \mathcal{L}}{\partial Y}[n,:]$ is the outer product
of two vectors, which yields a rank-1 matrix of the proper shape for $\frac{\partial \mathcal{L}_n}{\partial W}$.
Similarly, we have that $X[n,:] X[n,:]^T = \Vert X[n,:] \Vert_2^2$ is a 1x1 matrix, which can
be treated as a real number in all situations.

The conclusion for $\left\Vert \frac{\partial \mathcal{L}_n}{\partial b} \right\Vert_2^2$ follows automatically
from taking the norm of the corresponding expression in equation (\ref{eqn:individual-grads-Wb}).
Some more work is required for $\left\Vert \frac{\partial \mathcal{L}_n}{\partial W} \right\Vert_2^2$.
We will make use of the three following properties of matrix traces.
\begin{itemize}
\item $ \left\Vert A \right\Vert_\textrm{F}^2 = \Tr(AA^T) $
\item $ \Tr(ABC) = \Tr(BCA) = \Tr(CAB)$
\item $ \Tr(k A) = k \Tr(A)$ \hspace{1em} for $k\in\mathbb{R}$
\end{itemize}
We have that
\begin{eqnarray*}
\left\Vert \frac{\partial \mathcal{L}_n}{\partial W} \right\Vert_\textrm{F}^2 & = & \Tr\left( X[n,:]^T \frac{\partial \mathcal{L}_n}{\partial Y}[n,:]   \left(  X[n,:]^T \frac{\partial \mathcal{L}_n}{\partial Y}[n,:]   \right)^T \right) \\
& = & \Tr\left( X[n,:]^T \frac{\partial \mathcal{L}_n}{\partial Y}[n,:] \frac{\partial \mathcal{L}_n}{\partial Y}[n,:]^T  X[n,:] \right) \\
& = & \Tr\left( X[n,:] X[n,:]^T \frac{\partial \mathcal{L}_n}{\partial Y}[n,:] \frac{\partial \mathcal{L}_n}{\partial Y}[n,:]^T   \right) \\
& = & \left\Vert X[n,:] \right\Vert_2^2 \hspace{0.5em} \cdot \hspace{0.5em} \left\Vert \frac{\partial \mathcal{L}}{\partial Y} \right\Vert_2^2.
\end{eqnarray*}
\end{proof}

One might wonder why we are interested in computing the Frobenius norm
of the matrix $\frac{\partial \mathcal{L}_n}{\partial W}$ instead of its L2-norm.
The reason is that when running SGD we serialize all the parameters as a flat vector,
and it is the L2-norm of that vector that we want to compute.
We flatten the matrices, and the following equality reveals why this means
that we want the Frobenius norms of our matrix-shaped parameters :
\[
\left\Vert A \right\Vert_\textrm{F}^2 = \left\Vert A\textrm{.flatten( )} \right\Vert_2^2.
\]

\section{Distributed implementation of ISSGD}
\label{appsec:issgd-in-practice}

\subsection{Using only a subset of the weights}
\label{appsec:subset_weights}

Of the many hyperparameters that are available to adjust the behavior of the master and workers,
we have the possibility the use a staleness threshold that filters out all the $x_n$ candidates
whose corresponding $\tilde{\omega}_n$ have not been updated sufficiently recently.

For many of the experiments that we ran on SVHN, where we used a training set of roughly 570k samples,
with 3 workers, a staleness threshold of 4 seconds leads to 15\% of the probability weights to be kept.
The other 85\% are filtered out. This filtering is rather fair in that it does not
favor any sample a priori. Every $\tilde{\omega}_n$ stands equal chances of having been
recomputed last.

We have tried training without that staleness threshold and it is hard to see a difference.
Adding more workers naturally lowers the average staleness of probability weights, because more workers
can update them more frequently. If it were not the cost of communicating the
model parameters, we could argue that a sufficiently large number of workers would
simulate an oracle perfectly.

\subsection{Approximating $\left\Vert g^\textsc{true} \right\Vert_2^2$}
\label{appsec:gtrue_norm}

To report values of $\Tr(\Sigma(q))$, we need to be able to compute the actual expected
gradient over the whole training set. We refer to that quantity as the \emph{true gradient}
$g^\textsc{true}=\frac{1}{N}\sum g(x_n)$, but we never really compute it due to practical reasons.
This would entail reporting the gradient for each chunk of the training set and aggregating
everything. This is precisely the kind of thing that we avoid with ISSGD.

Instead we compute the gradients of the parameter for each minibatch, and we report
the L2-norm of those. We then average those values. This produces an upper-bound
to the actual value of $\left\Vert g^\textsc{true} \right\Vert_2$.

% TODO : This statement is vague. I have it in my notebook, and I'm not sure
%        that it's worth expanding on here, but it's a bit unsatisfying.

One important thing to note is that the equations (\ref{eqn:q-ideal-trcov}),
(\ref{eqn:q-unif-trcov}) and (\ref{eqn:q-stale-trcov}) each have the $\left\Vert g^\textsc{true} \right\Vert_2^2$
term, so any approximation of that term, provided that it is the same for all three,
will not alter the respective order of
$\Tr(\Sigma(q_\textsc{ideal})), \Tr(\Sigma(q_\textsc{stale})), \Tr(\Sigma(q_\textsc{unif}))$

Moreover, when are getting close to the end of the training, we should have that
$\left\Vert g^\textrm{true} \right\Vert_2$ is getting close to zero. That is, the gradient
is zero when we are close to an optimum. This does not meant that the individual gradients
are all zero, however. But when our upper-bound on $\left\Vert g^\textsc{true} \right\Vert_2$
is getting close to being insignificant, then we can tell for sure that our
values computed for the three $\Tr(\Sigma(q))$ are very close to their exact values.

%\subsection{Lipschitz constant of gradients}
%
%\stub{Discussion about how SAG cares about the Lipschitz constant, and mention the
%difficulty of dealing with bad examples (noisy examples) having large (but irrelevant)
%gradients. Refer to later section where we'll plot this experimentally.}

\subsection{Smoothing probability weights}
\label{appsec:smoothing-constant}

Sometimes we can end up with probability weights that fluctuate too rapidly.
This can lead to some problems in a situation where one training sample $x_n$
is assigned a small probability weight $\epsilon$, when compared to the other
probability weights. Things normally balance out because $x_n$ has a probability
proportional to $\epsilon$, and when it gets selected its gradient contribution
$g(x_n)$ gets scaled by $1 / \epsilon$. The resulting contribution is a gradient
of norm $\approx 1$.

However, when that gradient changes quickly (and probability weight along with it),
it is possible to get into a situation where the gradient computed on the master
is now much larger (due to the parameters having changed), and it still gets divided
by $\epsilon$ when selected. This does not affect the bias, but it affects the stability
of the method in the long term. When running for an indefinitely long period, it is
dangerous to having a time bomb in the algorithm that has a small probability of ruining everything.

To counter this effect, or just to make the training a bit smoother, we decided
to add a smoothing constant to all the probability weights $\tilde{\omega}_n$ before
normalizing them. The larger the constant, the more this will make ISSGD resemble regular SGD.
In the limit case where this constant is infinite, this turns exactly into regular SGD.

We had some ideas for using an adaptive method to compute this smoothing constant,
but this was not explored due to the large number of other hyperparameters to study.
One suggestion was to look at the entropy of the distribution of the $\left\{\omega_n\right\}_{n=0}^\infty$
that determine which training sample are going to be used. With a smoothing constant
sufficiently large, we can bring this entropy down to any target level (or down
to regular SGD when that constant is infinite).

\end{document}